\documentclass[aps,a4paper,pra,twocolumn,nofootinbib,balancelastpage,longbibliography,nobibnotes,numbers,floatfix]{revtex4-1}

\makeatletter
\renewcommand{\fnum@figure}{\textbf{\figurename~\thefigure}}
\makeatother
\renewcommand{\figurename}{Fig.}

\usepackage[verbose=true]{microtype}

\usepackage{amsmath}
\usepackage{amssymb}
\usepackage{txfonts}
\usepackage{mathtools}
\usepackage{upgreek}
\usepackage{bm}
\usepackage{siunitx}

\usepackage{xurl}
\usepackage[colorlinks,linktocpage,hypertexnames=false,pdfpagelabels,urlcolor=black,citecolor=blue,linkcolor=blue]{hyperref}

\usepackage{amsthm}
\usepackage{booktabs}
\usepackage{graphicx}
\usepackage{capt-of}
\usepackage{overpic}

\makeatletter
\renewcommand\section{%
  \@startsection{section}{1}{0pt}%
  {-\baselineskip}{.2\baselineskip}%
  {\normalfont\normalsize\bfseries\raggedright\color{blue}}}

\renewcommand\subsection{%
  \@startsection{subsection}{2}{0pt}%
  {-\baselineskip}{.1\baselineskip}%
  {\normalfont\normalsize\bfseries\raggedright}}

\renewcommand\paragraph{%
  \@startsection{paragraph}{4}{0pt}%
  {3.25ex\@plus 1ex\@minus .2ex}{-1em}%
  {\normalfont\normalsize\bfseries}}
\makeatother

\theoremstyle{thmstyleone}%
\newtheorem{theorem}{Theorem}%
\newtheorem{proposition}[theorem]{Proposition}%
\newtheorem{lemma}{Lemma}

\theoremstyle{thmstyletwo}%

\theoremstyle{thmstylethree}%

\newtheorem{corollary}{Corollary}

\newtheorem{myass}{Assumption}

\newtheorem{assumption}{Assumption}

\newcommand{\R}{\mathbb{R}}

\newcommand{\manifold}{\mathcal{M}}
\newcommand{\sphere}{\mathbb{S}}
\newcommand{\rbrac}[1]{\left(#1\right)}
\newcommand{\sbrac}[1]{\left[#1\right]}
\newcommand{\softmax}{\operatorname{softmax}}

\newcommand{\FAttn}{\operatorname{FAttn}}
\newcommand{\rownorm}{N_R}

\renewcommand{\d}{\mathrm{d}}
\newcommand{\Deltam}{\Delta_{\mathcal{M}}}

\newcommand{\Div}{\operatorname{div}}

\newcommand{\vol}{\mathrm{vol}}
\newcommand{\dman}{d_{\mathcal{M}}}
\newcommand{\xb}{\mathbf{x}}
\newcommand{\yb}{\mathbf{y}}
\newcommand{\zb}{\mathbf{z}}
\newcommand{\Wb}{\mathbf{W}}
\newcommand{\Ub}{\mathbf{U}}

\newcommand{\Sc}{\mathcal{S}}

\newcommand{\xib}{\bm{\upxi}}
\newcommand{\WQ}[1][h]{\Wb_{\mathrm{Q}}^{(#1)}}
\newcommand{\WK}[1][h]{\Wb_{\mathrm{K}}^{(#1)}}
\newcommand{\WV}[1][h]{\Wb_{\mathrm{V}}^{(#1)}}
\newcommand{\WO}[1][h]{\Wb_{\mathrm{O}}^{(#1)}}

\newcommand{\norm}[1]{\lVert #1 \rVert}
\newcommand{\abs}[1]{\lvert #1 \rvert}

\let\FNAbibfont\bibfont
\renewcommand{\bibfont}{\FNAbibfont\linespread{0.98}\selectfont}

\begin{document}

\title{Fractional neural attention for efficient multiscale sequence processing}

\author{Cheng Kevin Qu}
\email{chqu7424@gmail.com}
\thanks{These authors contributed equally to this work.}
\affiliation{School of Physics, University of Sydney, Sydney, NSW 2006, Australia}
\affiliation{Department of Applied Mathematics, University of Waterloo, Waterloo, ON N2L 3G1, Canada}
\affiliation{Perimeter Institute for Theoretical Physics, Waterloo, ON N2L 3G1, Canada}
\affiliation{Institute for Quantum Computing, University of Waterloo, Waterloo, ON N2L 3G1, Canada}
\author{Andrew Ly}
\email{anly2178@uni.sydney.edu.au}
\thanks{These authors contributed equally to this work.}
\affiliation{School of Physics, University of Sydney, Sydney, NSW, Australia}
\author{Pulin Gong}
\email{pulin.gong@sydney.edu.au}
\affiliation{School of Physics, University of Sydney, Sydney, NSW, Australia}

\begin{abstract}
Attention mechanisms are central to the computational power of Transformer models, yet the principles that make self-attention effective and how to extend them remain incompletely understood. Inspired by multiscale biological attention and dynamical systems theory, we introduce Fractional Neural Attention (FNA), a neuroscience-grounded framework for multiscale information processing. 
FNA models token interactions using kernels inspired by L\'evy diffusion, combining local interactions with occasional long-range jumps to capture dependencies across multiple scales. 
We provide a mechanistic account of this behavior by deriving a fractional diffusion equation governing FNA dynamics, showing that the resulting attention graphs exhibit larger spectral gaps and shorter effective path lengths, consistent with faster information mixing and increased expressivity. 
Empirically, FNA attains competitive text-classification performance even with a single layer and a single head and yields improvements in image processing and neural machine translation. 
Finally, using diffusion maps, we obtain low-dimensional representations of FNA attention that account for the distribution and intrinsic geometry of the token embeddings.
Together, these results connect self-attention to stochastic dynamics and geometry, offering an interpretable, neuroscience-motivated route to more efficient and expressive attention mechanisms.
\end{abstract}

\maketitle

\raggedbottom
\setlength{\textfloatsep}{12pt plus 2pt minus 2pt}
\setlength{\floatsep}{10pt plus 2pt minus 2pt}
\setlength{\intextsep}{10pt plus 2pt minus 2pt}
\setlength{\dbltextfloatsep}{12pt plus 2pt minus 2pt}
\setlength{\dblfloatsep}{10pt plus 2pt minus 2pt}

Transformers \cite{Vaswani2017} have delivered striking results on sequential data, spanning natural language \cite{Liu2019, Zhu2020} and time series \cite{Li2019,Zhou2022}, and have been successfully adapted to vision \cite{Dosovitskiy2020} and graph learning \cite{Yang2021}.
Their pervasive success is widely attributed to the self-attention mechanism that models the relationships between tokens.
Yet the principles that make self-attention effective remain largely unclear. 
Recent studies have recast token embeddings as spatial variables and layers as a time coordinate, yielding a continuous-time view of the limiting dynamics of self-attention \cite{Lu2019, Sander2022}. Notably, a bistochastic formulation of self-attention reduces to classical Brownian diffusion \cite{Sander2022}.
Furthermore, local-global hybrids of attention have been found to consistently improve long-sequence processing \cite{Zaheer2020, beltagy2020longformer}. 
This two-scale structure has largely been engineered through ad hoc manipulation of sparsification windows, combining sliding windows with special global tokens or random links. 
A principled formulation of self-attention that inherently entails multiscale dynamics has been missing.

Neuroscience offers an instructive perspective. 
It has been shown that attentional spotlights exhibit local shifts interspersed with long-range jumps, with step sizes following heavy-tailed, power-law statistics \cite{Chen2022} (Fig.~\ref{fig:schematic}). 
These attention-sampling dynamics unfold across multiple nested scales, and are well-characterized as a type of anomalous diffusion---fractional L\'evy motion.  
Such properties account for diverse neurophysiological observations of attention sampling \cite{Chen2022} and confer computational advantages \cite{Qi2022}, enabling efficient sampling of natural images, which themselves possess nested multiscale structure and heavy-tailed dynamics \cite{Munn2018,Ruderman1994,Field87}. 
In contrast, Brownian-motion-based sampling is predominantly local (Fig.~\ref{fig:schematic}), due to the lack of long jumps. 
Although neuroscience-inspired artificial intelligence (NeuroAI) is widely expected to be crucial for future advances \cite{Hassabis2017, Zador2023}, a neurally grounded attention rule that operationalizes realistic, multiscale computation within Transformer architectures remains unexplored.

\begin{figure*}[!t]
    \centering
    \includegraphics[scale=0.8]{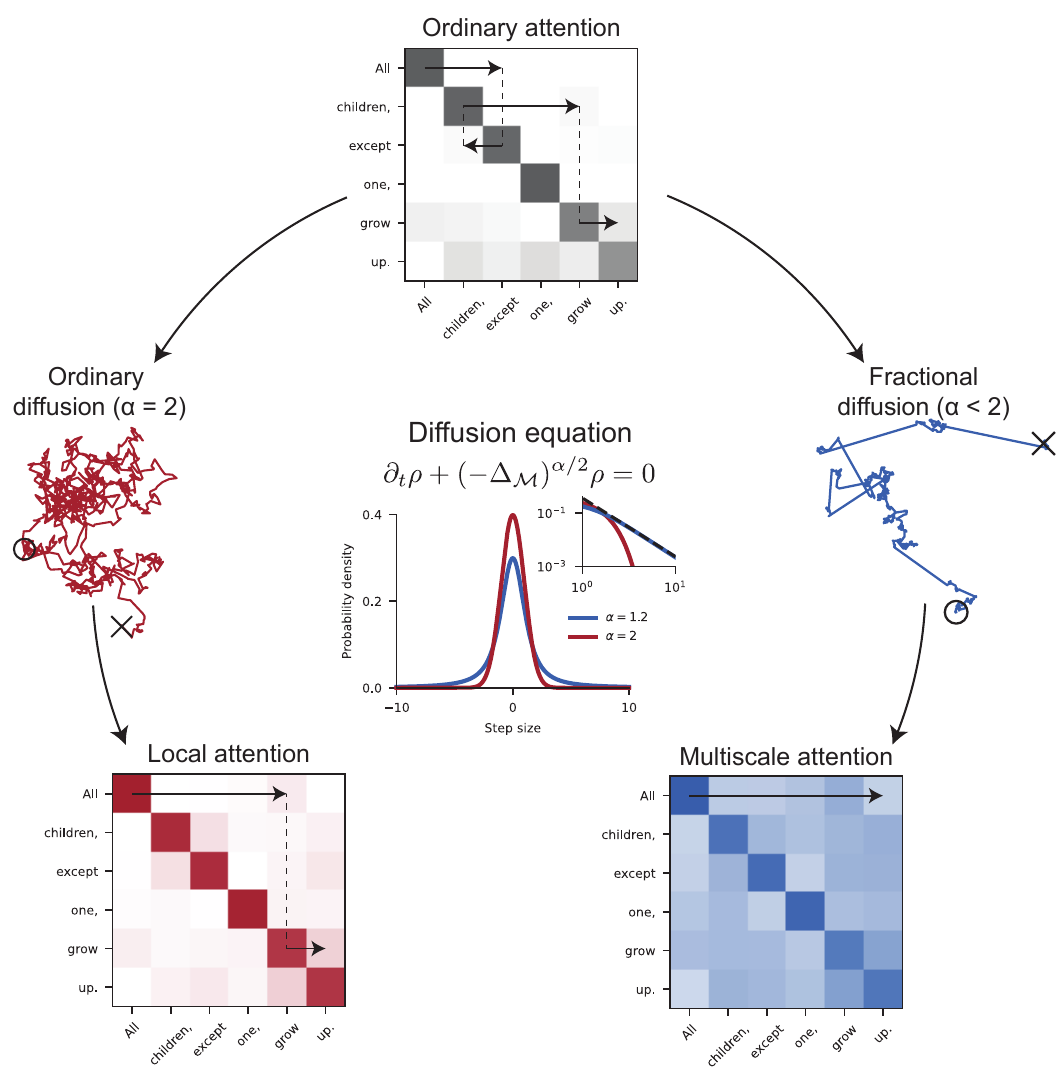} 
    \caption{\textbf{From fractional diffusion to attention.} 
    The fractional diffusion equation describes the density evolution of L\'evy processes for $\alpha < 2$. 
    When $\alpha = 2$, it reduces to the classical diffusion equation corresponding to Brownian motion. 
    The associated step-size probability densities are heavy-tailed (i.e., multiscale) for $\alpha < 2$ and Gaussian for $\alpha = 2$.
    Inspired by the role of fractional diffusion in neurobiological attention \cite{Chen2022}, we incorporate these dynamics into Transformer self-attention to obtain fractional neural attention (FNA). 
    The schematic shows how multiscale FNA ($\alpha < 2$) enhances the expressivity of attention, enabling the first word of a sentence (e.g., ``All children, except one, grow up.''---Peter Pan) to attend to the last with a single attention step. 
    In contrast, local attention ($\alpha = 2$) and standard attention typically require multiple steps. Steps are represented by arrows within the attention matrices.
    }
    \label{fig:schematic}
\end{figure*}

In this study, we introduce a neuroscience-inspired attention mechanism, which we term fractional neural attention (FNA), that integrates L{\'e}vy diffusion into self-attention through a dynamical-systems formulation. 
FNA intrinsically realizes simultaneous short- and long-range interactions among tokens across multiple scales, yielding greater expressivity and faster information mixing, thus advancing the foundational capacity of Transformers beyond efficiency-oriented designs. 
To elucidate the theoretical framework of FNA, we first prove that FNA’s second-order dynamics are governed by the fractional Laplacian (Fig.~\ref{fig:schematic}) \cite{Lischke2020}, a scale-free pseudo-differential operator. 
Its order $\alpha\in[1,2]$ controls the dynamics of the associated symmetric stable L\'evy process \cite{Applebaum_2009}.  
This yields an interpretable, geometry-aware account of attention as diffusion on a manifold, enabling the canonical use of diffusion maps for dimensionality reduction to obtain meaningful representations of the attention weights \cite{Coifman2006, Antil2021}. 
The parameter $\alpha$ controls locality versus non-locality, recovering the local Brownian limit at $\alpha = 2$ and inducing power-law tails for $\alpha < 2$ that support multiscale (i.e., scale-free) interactions among tokens in embedding space.

We then explain FNA's computational advantages using a network-theoretic analysis of the induced attention networks. 
Relative to standard self-attention, FNA produces larger spectral gaps and shorter path lengths, enabling faster information mixing and greater representational capacity. 
This analysis explains why FNA attains strong performance with substantially fewer layers---sometimes a single layer---than conventional Transformers. 
Finally, across text (IMDb), vision (CIFAR-10 with ViTs \cite{Dosovitskiy2020}) and machine translation (Multi30K En–De \cite{Elliott2016}), multiscale FNA matches or exceeds Transformer baselines as well as model variants derived from the continuous-time formulation of self-attention \cite{Sander2022}.
Together, these results establish a principled, multiscale attention mechanism that connects insights from neuroscience and dynamical systems to advance AI.

\section*{Theoretical framework of fractional neural attention (FNA)}

We develop a theoretical framework for FNA through three steps: First, we dissect the self-attention mechanism of the vanilla Transformer as a continuous-time dynamical system, uncovering a fundamental connection to local diffusive behavior. 
Second, we generalize this to non-local, fractional diffusion by exploiting the fractional Laplacian operator, the infinitesimal generator of the symmetric $\alpha$-stable L{\'e}vy processes. 
Third, we integrate this fractional diffusion into attention to formulate FNA.

\subsection*{Self-attention as a dynamical system.} 
The original Transformer model employs self-attention in a feed-forward manner as the central component for information processing \cite{Bahdanau2014,Vaswani2017,Lin2022}.
To summarize, the self-attention block in the $(l+1)^{\text{th}}$ layer receives as input a sequence of $n$ tokens embedded in a manifold such as Euclidean space, i.e., $\mathbf{X}^l \coloneqq (\mathbf{x}_1^l, \mathbf{x}_2^l, \cdots, \mathbf{x}_n^l) \in \mathbb{R}^{d \times n}$ where $\mathbf{x}_i^l \in \R^d$ for each $i$. 
It projects these embedding vectors to queries, keys and values: for all $i$, $\mathbf{q}_i = \mathbf{W}_{\mathrm{Q}} \mathbf{x}_i$, $\mathbf{k}_i = \mathbf{W}_{\mathrm{K}} \mathbf{x}_i$ and $\mathbf{v}_i = \mathbf{W}_{\mathrm{V}} \mathbf{x}_i$, respectively, where $\mathbf{W}_{\mathrm{Q}}, \mathbf{W}_{\mathrm{K}} \in \mathbb{R}^{d_k \times d}$ and $\mathbf{W}_{\mathrm{V}} \in \mathbb{R}^{d_v \times d}$ are learnable projection matrices. 
The self-attention block updates each token representation by aggregating information across the entire input sequence, producing the output embedding \cite{Vuckovic2020,Vuckovic2021,Sander2022,Geshkovski2023b}, giving
\begin{equation} \label{eq:sa_resnet}
\mathbf{x}_i^{l+1} = \mathbf{x}_i^{l} + \sum_{j = 1}^n A_{ij}^{\mathrm{DP}} \mathbf{W}_{\mathrm{V}} \mathbf{x}_j^l,
\end{equation} 
where $\mathbf{A}^\mathrm{DP} \coloneqq \softmax(\mathbf{C}^\mathrm{DP}) = N_R(\exp[\mathbf{C}^\mathrm{DP}]) \in \mathbb{R}^{n \times n}$.
The superscript DP represents dot-product (self-attention).
The exponential function $\exp[\cdot]$ is applied element-wise to the similarity matrix $\mathbf{C}^\mathrm{DP} \in \mathbb{R}^{n \times n}$, whose entries are calculated as query--key dot products $C_{ij}^\mathrm{DP} \coloneqq \mathbf{q}_i^{\top} \mathbf{k}_j = (\mathbf{W}_{\mathrm{Q}} \mathbf{x}_i^l)^{\top} \mathbf{W}_{\mathrm{K}} \mathbf{x}_j^l$, and $\rownorm(\cdot)$ denotes row normalization. 
The resulting attention weight matrix $\mathbf{A}$ is right-stochastic: each row sums to 1, and each entry $A_{ij}^{\mathrm{DP}}$ can be interpreted as the normalized weight assigned by the $i^{\text{th}}$ query to the $j^{\text{th}}$ key, applied to the value $\mathbf{v}_j \coloneqq \mathbf{W}_{\mathrm{V}} \mathbf{x}_j^l$ when forming $\mathbf{x}_i^{l+1}$.
In other words, the attention weights encode single-step transition probabilities over the token sequence.
For notational convenience, we collect queries, keys and values into matrices $\mathbf{Q} \coloneqq (\mathbf{q}_1, \cdots, \mathbf{q}_n)$, $\mathbf{K} \coloneqq (\mathbf{k}_1, \cdots, \mathbf{k}_n)$ and $\mathbf{V} \coloneqq (\mathbf{v}_1, \cdots, \mathbf{v}_n)$, respectively.
In addition to the self-attention mechanism, the residual connections are a crucial component of the Transformer architecture. 
They prevent representational degradation within the attention block; without them, the model output suffers from rank-collapse at a double-exponential rate with respect to depth \cite{Dong2021}.
To guarantee the well-definedness of equation~\eqref{eq:sa_resnet}, we assume $d_k = d_v = d$ \cite{Vaswani2017}. 

Following recent studies \cite{Lu2019,Sander2022,Geshkovski2023b}, we now interpret the self-attention mechanism as a dynamical system.
The residual form of equation~\eqref{eq:sa_resnet} corresponds to an Euler discretization with a step size of 1 of the ordinary differential equation (ODE)
\begin{equation} \label{eq:resnet_ode}
\dot{\mathbf{x}}(t) = T(\mathbf{x}(t)),
\end{equation}
where the token embeddings $\mathbf{x}_i$ are treated as spatial variables and the layer $l$ in equation~\eqref{eq:sa_resnet} is viewed as a time variable $t$.
This continuous-time perspective parallels formulations in the framework of neural ODEs \cite{Chen2018}.
To highlight this interpretation, we suppress the superscript $l$ for $\mathbf{X}$ (and $\mathbf{x}$) when we use $t$.
Furthermore, we construct the attention map $T$ as
\begin{equation} \label{eq:attention_map_main}
T_{\rho,\varepsilon}(\mathbf{x}) = \int_{\R^d} k_{\varepsilon}^{\mathrm{DP}}(\mathbf{x}, \mathbf{y}) \mathbf{W}_{\mathrm{V}} \mathbf{y} \rho(\mathbf{y}) \d \mathbf{y}.
\end{equation}
The associated DP attention kernel is
\begin{equation}
k_{\varepsilon}^{\mathrm{DP}}(\mathbf{x}, \mathbf{y}) \coloneqq \frac{\exp(\mathbf{x}^{\top} \mathbf{W}_{\mathrm{Q}}^{\top} \mathbf{W}_{\mathrm{K}} \mathbf{y} / \varepsilon)}{\displaystyle\int_{\R^d} \exp(\mathbf{x}^{\top} \mathbf{W}_{\mathrm{Q}}^{\top} \mathbf{W}_{\mathrm{K}} \mathbf{x}'/\varepsilon) \rho(\mathbf{x}') \d \mathbf{x}'},
\end{equation}

where $\varepsilon>0$ is the kernel bandwidth parameter, $t$ is the evolution time in equation~\eqref{eq:resnet_ode}, and $\rho$ is the distribution of the token embeddings.
Each token $\mathbf{x}_i$ can be viewed as a particle whose position evolves according to equation~\eqref{eq:resnet_ode}. 
We thus use the terms token and particle interchangeably throughout the remainder of the text \cite{Lu2019}.
Implicitly, we take the infinite-particle limit $n \rightarrow \infty$ to obtain the integral expression of the summation term in equation~\eqref{eq:sa_resnet}.
The attention map $T$ and the attention kernel $k$ depend on $\rho$, the probability density supported on $\mathbb{R}^d$ from which the embeddings $\mathbf{y}$ are sampled. 
Using the dot-product kernel $k_1^{\mathrm{DP}}$ and the empirical measure $\rho = \sum_{i=1}^n \delta_{\mathbf{x}_i}^{(d)}$ in equation~\eqref{eq:attention_map_main}, a unit Euler step in equation~\eqref{eq:resnet_ode} recovers the original dot-product update in equation~\eqref{eq:sa_resnet} \cite{Vuckovic2020,Vuckovic2021,Koubbi2024}.
For the small-bandwidth limit of squared-Euclidean-distance attention \cite{Sander2022}, let $T_{\rho,\varepsilon}^{\mathrm{L2}}$ denote the map in equation~\eqref{eq:attention_map_main} with the row-normalized kernel \[
k_{\varepsilon}^{\mathrm{L2}}(\mathbf{x},\mathbf{y}) \coloneqq
\frac{\exp(-\Vert \mathbf{W}_{\mathrm{Q}}\mathbf{x}-\mathbf{W}_{\mathrm{K}}\mathbf{y}\Vert^2/(2\varepsilon))}
{\int_{\R^d}\exp(-\Vert \mathbf{W}_{\mathrm{Q}}\mathbf{x}-\mathbf{W}_{\mathrm{K}}\mathbf{x}'\Vert^2/(2\varepsilon))\rho(\mathbf{x}')\d\mathbf{x}'}.
\]
Under appropriate projection assumptions, with invertible $\mathbf{W}_{\mathrm{Q}}$ and $\mathbf{W}_{\mathrm{K}}$, and $\mathbf{W}_{\mathrm{Q}}^{\top}\mathbf{W}_{\mathrm{K}}=\mathbf{W}_{\mathrm{K}}^{\top}\mathbf{W}_{\mathrm{Q}}=-\mathbf{W}_{\mathrm{V}}$, the recentered velocity field is
$\widetilde{T}_{\rho,\varepsilon}(\mathbf{x})\coloneqq\varepsilon^{-1}(T_{\rho,\varepsilon}^{\mathrm{L2}}(\mathbf{x})+\mathbf{W}_{\mathrm{Q}}^{\top}\mathbf{W}_{\mathrm{Q}}\mathbf{x})$.
Writing $\widetilde{T}_{\rho,0}=\lim_{\varepsilon\downarrow0}\widetilde{T}_{\rho,\varepsilon}$, the corresponding second order dynamics is described by the continuity equation:
\begin{equation} \label{eq:continuity_eq_1}
\partial_t \rho + \Div(\rho \widetilde{T}_{\rho, 0}) = 0, \quad (t,\mathbf{x}) \in \R^+ \times \R^d. 
\end{equation}
The continuity equation provides a deeper understanding of how the probability distribution of particles evolves.
For instance, a recent work studied a bistochastic form of self-attention, in which an iterative scheme of column and row normalizations is applied to the attention matrix $\mathbf{A}$ \cite{Sander2022}.
For this bistochastic construction, the recentered velocity is $\varepsilon^{-1}(T_{\rho,\varepsilon}^{\mathrm{SINK}}(\mathbf{x})-\mathbf{W}_{\mathrm{V}}\mathbf{x})$, where $T_{\rho,\varepsilon}^{\mathrm{SINK}}$ uses the bistochastically normalized kernel. 
Under the same regularity and projection assumptions for deriving equation~\eqref{eq:continuity_eq_1}, the corresponding continuity equation reduces to the spatially local diffusion equation \cite{Renardy2004}:
\begin{align} \label{eq:heat_eq}
\partial_t \rho  - \Delta \rho = 0,   
\end{align}
where $\Delta$ is the standard Laplacian operator and $\rho \coloneqq \rho(t,\mathbf{x})$ is the particle density at time $t$ and position $\mathbf{x}$.
With modern deep models such as GPT-3, employing up to 96 layers \cite{Brown2020}, understanding the layer-wise (continuum) limit of self-attention in equation~\eqref{eq:sa_resnet} becomes increasingly important.
Adopting a dynamical-systems perspective clarifies how the probability distribution governing the particles evolves across layers.

\subsection*{Fractional diffusion.} 
We now establish the spatially fractional generalization of the diffusion equation in order to guide its integration within self-attention.
The generalization of equation~\eqref{eq:heat_eq} is the fractional diffusion equation on a Riemannian manifold $\manifold$:
\begin{equation} \label{eq:frac_heat_eq}
\partial_t \rho + (-\Deltam)^{\alpha/2} \rho = 0,
\end{equation}
where the fractional Laplacian $(-\Deltam)^{\alpha/2}$ is a non-local, pseudo-differential operator; $\alpha \in (0, d_{\mathcal{M}}+1]$ is its fractional order and $d_{\mathcal{M}}$ is the dimension of the manifold $\mathcal{M}$.
We consider two manifolds, $\manifold = \mathbb{R}^d$ and $\sphere^{d-1}$, as both will play important theoretical and practical roles later on.
There are several equivalent definitions of the fractional Laplacian \cite{Kwasnicki2017, Lischke2020}. 
In particular, for $\manifold = \R^d$ and $\alpha \in (1,2)$, the fractional Laplacian can be defined as a singular integral:
\begin{equation} \label{eq:frac_laplacian_euclidean}
(-\Delta)^{\alpha/2} u(\mathbf{x})
= C_{d, \alpha} \lim _{\varepsilon \downarrow 0} \int_{\mathbb{R}^d \backslash B(\mathbf{x}, \varepsilon)} \frac{u(\mathbf{x}) - u(\mathbf{y})}{\Vert\mathbf{x}-\mathbf{y}\Vert^{d+\alpha}} \d \mathbf{y},
\end{equation}

where $\Delta \coloneqq \Delta_{\R^d}$, $C_{d, \alpha} \coloneqq \frac{ \alpha 2^{\alpha-1} \Gamma\left(\frac{d+\alpha}{2}\right)}{\pi^{d/2} \vert \Gamma(1-\alpha/2) \vert}$ is a normalization constant, $\Gamma$ is the Gamma function and $B(\mathbf{x}, \varepsilon)$ is the open ball of radius $\varepsilon$ centered at $\mathbf{x}$.
When $\alpha = 2$, the operator reduces to the negative standard (local) Laplacian: the (negative) value of $\Delta f(\mathbf{x})$ depends only on an arbitrarily small neighborhood of $f$ around $\mathbf{x}$. 
In contrast, for $\alpha \in (0,2)$, the value of the fractional Laplacian $(-\Delta)^{\alpha/2} f(\mathbf{x})$ depends on the entire domain according to the power-law weighting in the integral of equation~\eqref{eq:frac_laplacian_euclidean}. 
See ``Methods'' for a spectral definition of the fractional Laplacian for compact manifolds without boundary.

For the symmetric isotropic settings considered here, the fundamental solution of equation~\eqref{eq:frac_heat_eq} is a heat kernel $k_t(\mathbf{x},\mathbf{y})$. Under the conditions in ``Methods'', an $\tilde{\alpha}$-scale invariant heat kernel admits small-time bounds in terms of the geodesic distance $d_g(\cdot,\cdot)$ on $\mathcal{M}$, defined in equation~\eqref{eq:g_dist}:
\begin{equation} \label{eq:beta_scale_kernel}
\frac{c_1}{t^{\dman / \tilde{\alpha}}} \Phi\left(C_1 \frac{d_g(\mathbf{x},\mathbf{y})}{t^{1 / \tilde{\alpha}}}\right) \leq k_t(\mathbf{x}, \mathbf{y}) \leq \frac{c_2}{t^{\dman / \tilde{\alpha}}} \Phi\left(C_2 \frac{d_g(\mathbf{x},\mathbf{y})}{t^{1 / \tilde{\alpha}}}\right),
\end{equation}

where $c_1, c_2, C_1, C_2$ are all positive constants and the decay function $\Phi: [0, \infty) \rightarrow [0, \infty)$ is monotonically decreasing.
Such heat kernels $k$ defined by equation~\eqref{eq:kernel_positivity}--\eqref{eq:kernel_identity} in ``Methods'' exhibit a dichotomy \cite{Grigor2008,Grigoryan2018}:
\begin{enumerate}
\item If  $\tilde{\alpha} \in [2,d_{\mathcal{M}}+1]$,  then $k$ is local and its decay is bounded by the exponential
\begin{equation} \label{eq:hk_dichotomy_1}
\hspace{-0.4cm} \Phi_{\tilde{\alpha}}(z) = \exp\rbrac{-z^{\frac{\tilde{\alpha}}{\tilde{\alpha}-1}}};
\end{equation}
or 
\item If $\tilde{\alpha} \in (0,2)$, then $k$ is non-local and its decay is bounded by the power law
\begin{equation} \label{eq:hk_dichotomy_2}
\Phi_{\tilde{\alpha}}(z) = (1+z)^{-(d_{\mathcal{M}}+\tilde{\alpha})}.
\end{equation}
\end{enumerate}

For $\tilde{\alpha} \in [1,2]$ and $\mathcal{M}=\mathbb{R}^d$ (i.e., $d_g(\mathbf{x},\mathbf{y})=\Vert \mathbf{x}-\mathbf{y}\Vert$), there is a direct correspondence between $\alpha$ and $\tilde{\alpha}$, as illustrated in Fig.~\ref{fig:schematic}, unifying the fractional Laplacian, $\tilde{\alpha}$-scale invariant heat kernels, and their corresponding stochastic processes. 
In this setting, the class of isotropic symmetric $\alpha$-stable ($S\alpha S$) L\'evy processes is well-defined (see ``Methods''), with characteristic function 
\begin{equation}\label{eq:stable_vector_char_func}
\mathbb{E}\!\left[e^{i\bm{\upxi}^{\top}\mathbf{L}_t}\right]
=
\exp\!\left(-t\Vert\bm{\upxi}\Vert^\alpha\right),
\end{equation}
a multidimensional generalization of equation~\eqref{eq:stable_dist_char_func}.
For brevity, we refer to these processes as L\'evy processes throughout the remainder of the manuscript. 
Accordingly, we restrict our framework to this regime and henceforth identify $\tilde{\alpha}=\alpha\in[1,2]$.

See ``Methods'' for further details on the fractional Laplacian, their connection to the $\alpha$-scale invariant heat kernel (equation~\eqref{eq:beta_scale_kernel}) and its generalization to non-Euclidean manifolds.

We next exemplify the dichotomy of local and multiscale dynamics encompassed within the fractional diffusion equation.
For $\alpha = 2$, where equation~\eqref{eq:frac_heat_eq} reduces to equation~\eqref{eq:heat_eq}, the solution is the well-known Gaussian heat kernel:
\begin{equation} \label{eq:local_heat_kernel_euclidean}
k_t(\mathbf{x},\mathbf{y}) \coloneqq \frac{1}{(4\pi t)^{d / 2}} e^{-\frac{\Vert \mathbf{x} - \mathbf{y} \Vert^2}{4 t}},
\end{equation}
which decays exponentially in the squared Euclidean distance.
This describes the probability density associated with the classical Brownian motion (Fig.~\ref{fig:schematic}). 
In contrast, for $\alpha < 2$, the solution yields disparate characteristics.
For example, it simplifies to the well-known Poisson kernel when $\alpha = 1$:
\begin{equation}
k_t(\mathbf{x}, \mathbf{y})=\frac{\overline{c}_d}{t^d}\left(1+\frac{\Vert \mathbf{x}-\mathbf{y} \Vert^2}{t^2}\right)^{-\frac{d+1}{2}},
\end{equation}
where $\overline{c}_d=\Gamma\left(\frac{d+1}{2}\right) / \pi^{(d+1) / 2}$.  
Such power-law heat kernels describe the probability densities of L\'evy processes. 
The increments of L\'evy processes exhibit an asymptotic power-law decay, so large jumps occur with far higher probability than under Brownian motion (Fig.~\ref{fig:schematic}). 
Such heavy-tailed dynamics are efficient for neural circuits implementing probabilistic computation \cite{Qi2022} and for spatial-visual attention sampling \cite{Chen2022}.
They are also consistent with observed efficient search strategies in animal foraging \cite{Viswanathan1999,Radons2008}.

\subsection*{Fractional neural attention.} \label{sec:FNA}
We implement FNA through the following steps:
First, for all pairs of tokens $\mathbf{x}_i, \mathbf{x}_j \in \mathbb{R}^d$ in the sequence, we calculate the FNA score as
\begin{equation} \label{eq:frac_attn_score}
C_{ij} = \Phi_{\alpha}\rbrac{ \frac{\Vert \mathbf{W}_{\mathrm{Q}} \mathbf{x}_i - \mathbf{W}_{\mathrm{K}} \mathbf{x}_j \Vert}{\kappa}}.
\end{equation}
Here $\kappa$ is a distance scaling constant. 
We use equations~\eqref{eq:hk_dichotomy_1}--\eqref{eq:hk_dichotomy_2} for $\Phi_{\alpha}$ to define the FNA affinity. The heat-kernel bounds in equation~\eqref{eq:beta_scale_kernel} motivate these profiles and describe their decay.  

We emphasize a key change to conventional self-attention:
To exploit $\Phi_{\alpha}$, we replace the dot-product between the queries and keys with their Euclidean (or $L^2$) distance.
This operation decouples the attention score from the  softmax’s implicit exponential form, which has a fixed scale, and enables power-law (scale-free) attention when $\alpha < 2$.
Distance-based self-attention has been previously proposed, primarily to establish Lipschitz continuity \cite{Kim2021}; our construction extends this approach and recovers a related Gaussian form when $\alpha = 2$.

Second, we compute the output of the FNA block as 
\begin{equation} \label{eq:frac_attn_weights}
\FAttn(\mathbf{X}) \coloneqq \mathbf{V} \mathbf{A}^{\top}, \ \mathbf{A} \coloneqq \rownorm(\mathbf{C}). 
\end{equation}   
Here, we reserve $\mathbf{C} \coloneqq \mathbf{C}^{\mathrm{FNA}}$ (equation~\eqref{eq:frac_attn_score}) and $\mathbf{A} \coloneqq \mathbf{A}^{\mathrm{FNA}}$ for FNA, unless stated otherwise.

For $\alpha = 2$, equation~\eqref{eq:hk_dichotomy_1} reduces to the Gaussian kernel. 
For Euclidean attention, we set $\kappa=\sqrt{d_h}$, where $d_h\coloneqq d/H$ is the dimension of each attention head, giving the affinity $\exp(-\lVert\mathbf q_i-\mathbf k_j\rVert^2/d_h)$.
Under the standard assumption that the coordinates of the queries and keys have $O(1)$ variance \cite{Hron2020,Xiong2020,Noci2022}, the typical Euclidean distance scales as $\lVert \mathbf q_i-\mathbf k_j\rVert=O(\sqrt{d_h})$, so this normalization keeps the kernel exponent $O(1)$ as the head dimension increases.
This scaling plays a role analogous to the $1/\sqrt{d_h}$ factor used in the standard dot-product attention \cite{Vaswani2017}. 
For $\alpha < 2$, a different scaling of $\kappa$ is required because the non-local kernel in equation~\eqref{eq:hk_dichotomy_2} depends explicitly on the manifold dimension through the exponent $\dman + \alpha$.
Since the attention kernel is evaluated independently within each head, we set $\dman = d_h$.
The value of $\alpha$, together with the associated choice of $\kappa$, is specified prior to training as a hyperparameter.
We postulate the ansatz that $\kappa$ is $d_h$-dependent and derive three regimes of multiscale attention under $\lVert \mathbf q_i-\mathbf k_j\rVert=O(\sqrt{d_h})$---uniform, critical and sharp, that have been considered previously in \cite{Giuseppe2026}.
See ``Methods'' for further details regarding the scaling laws and implementation of FNA. 
Note that the multiscale case (equation~\eqref{eq:hk_dichotomy_2}) shares the same asymptotic cost as its local counterpart (equation~\eqref{eq:hk_dichotomy_1}).
Moreover, $L^2$-normalization can also be applied to the embeddings prior to the attention map, naturally mapping them onto $\manifold = \sphere^{d-1}$ (see Appendix~\ref{supp_sec:spherical_fna}).

\subsection*{Operator limit of FNA.}
We now derive the Euclidean continuum generator of FNA and identify its dependence on the fractional Laplacian and the token density. 
Paralleling the dynamical-systems formulation of ordinary attention (equation~\eqref{eq:attention_map_main}), the infinitesimal generator of FNA dynamics is 
\begin{equation}  \label{eq:inf_gen_FNA}
    \mathcal{H}^{\mathrm{FNA}}[f] \coloneqq \lim_{t \rightarrow 0} \frac{T_{\rho,t}^{\mathrm{FNA}}[f] - f}{t},
\end{equation}
where $f \in C_c^{\infty}(\R^d)$ is a generic test function,
\begin{equation} \label{eq:semigroup}
T_{\rho,\varepsilon}^{\mathrm{FNA}}[f](\mathbf{x}) = \int_{\mathbb{R}^d} k_{\varepsilon}^{\mathrm{FNA}}(\mathbf{x},\mathbf{y}) f(\mathbf{y}) \rho(\mathbf{y}) \d \mathbf{y}   
\end{equation}
is the attention map and 
\begin{equation} \label{eq:fna_kernel}
k_{\varepsilon}^{\mathrm{FNA}}(\mathbf{x}, \mathbf{y}) \coloneqq \frac{\Phi_{\alpha}(\Vert \mathbf{W}_{\mathrm{Q}} \mathbf{x} - \mathbf{W}_{\mathrm{K}} \mathbf{y} \Vert / \varepsilon^{1/\alpha})}{\displaystyle\int_{\R^d} \Phi_{\alpha}(\Vert \mathbf{W}_{\mathrm{Q}} \mathbf{x} - \mathbf{W}_{\mathrm{K}} \mathbf{x}' \Vert / \varepsilon^{1/\alpha})  \rho(\mathbf{x}') \d \mathbf{x}'}
\end{equation}

is the FNA kernel. 
The connection between the infinitesimal generator $\mathcal{H}^{\text{FNA}}$ and the fractional Laplacian $(-\Delta)^{\alpha/2}$ relies on orthogonal and tied query and key projections (Theorem~\ref{thm:fna_inf_gen}). 
The attention map (equation~\eqref{eq:semigroup}) represents the spatially continuous analog of the FNA weights (equation~\eqref{eq:frac_attn_weights}) in the limit of $n \rightarrow \infty$ (see Appendices~\ref{supp_sub_sec:diffusion_regimes} and \ref{supp_sub_sec:density_normalization}).
The output of the FNA block in equation~\eqref{eq:frac_attn_weights} is realized under the choice $f(\mathbf{x}) = \mathbf{W}_{\mathrm{V}} \mathbf{x}$, but we retain the current formulation for generality.
We proceed by showing how to guarantee these properties in the practical implementation of FNA.

Both the original attention \cite{Vaswani2017} and our FNA apply linear projections of embedding vectors $\mathbf{x}_i$ to obtain queries and keys: $\mathbf{q}_i = \iota_Q(\mathbf{x}_i) = \mathbf{W}_{\mathrm{Q}} \mathbf{x}_i$ and $\mathbf{k}_i = \iota_K(\mathbf{x}_i) = \mathbf{W}_{\mathrm{K}} \mathbf{x}_i$, respectively.
For arbitrary $\mathbf{W}$, the projection $\iota$ is generally not an isometry because the geodesic distance between an arbitrary pair of tokens is not necessarily preserved
(see Appendix~\ref{supp_sec:multihead_fna}).
For $\manifold = \R^d$, an isometric projection $\iota$ can be realized under the following assumptions: 

\begin{myass}[orthogonality] \label{ass:orthogonality}

Assume single-head attention $H = 1$ and orthogonal query and key projection matrices $\mathbf{W}_{\mathrm{Q,K}} \in O(d)$.
The orthogonality of a weight matrix can be enforced during training through constrained optimization methods for matrix manifolds \cite{LezcanoCasado2019}.
This condition has also been used in other Transformer architectures to maintain the token embedding manifold structure and thus preserve the injectivity of the model input-output map \cite{sengupta2023}.

\end{myass}
\smallskip

\begin{myass}[symmetry]\label{ass:symmetry}

Assume weight-tying between the queries and keys, i.e., $\mathbf{W}_{\mathrm{Q}} = \mathbf{W}_{\mathrm{K}}$, which implies $\mathbf{Q} = \mathbf{K}$ for every input. 
This condition has motivated the development of new architectures by linking the attention kernel $k^{\mathrm{SA}}$ to the Gaussian kernel in \cite{Chen2021}.
In recent studies \cite{Geshkovski2023b}, various similar assumptions on the keys and queries have been applied to understand self-attention, including $\mathbf{Q}^{\top} \mathbf{K} \succ 0$ (positive-definiteness) and $\mathbf{Q}^{\top} \mathbf{K} = \mathbf{I}_d$.

\end{myass} 
\smallskip

When both assumptions are satisfied, a meaningful parameter reduction of the original self-attention (equation~\eqref{eq:sa_resnet}) and FNA is possible through the following proposition (see ``Methods'' for proof).

\begin{proposition} \label{prop:distance_preserving}
Assume $\manifold = \R^d$ or $\sphere^{d - 1}$, single-head attention $H = 1$ and $\mathbf{W}_{\mathrm{Q,K}} \in O(d)$ (i.e., Assumption~\ref{ass:orthogonality}).
Then, self-attention (equation~\eqref{eq:sa_resnet}) and FNA (equation~\eqref{eq:frac_attn_score}) with query and key projection matrices $\mathbf{W}_{\mathrm{Q}}$ and $\mathbf{W}_{\mathrm{K}}$ are equivalent to using an alternative pair of query/key projection matrices $\tilde{\mathbf{W}}_{\mathrm{Q}} = \mathbf{I}_d$ and $\tilde{\mathbf{W}}_{\mathrm{K}} = \mathbf{W}_{\mathrm{Q}}^{\top} \mathbf{W}_{\mathrm{K}} \in O(d)$.
When Assumption~\ref{ass:symmetry} is further satisfied, $\tilde{\mathbf{W}}_{\mathrm{Q}} = \tilde{\mathbf{W}}_{\mathrm{K}} = \mathbf{I}_d$.
\end{proposition} 
\smallskip

The following theorem gives the Euclidean limit of the normalized FNA operator in equation~\eqref{eq:inf_gen_FNA}. 

\begin{theorem}[Infinitesimal generator for FNA] \label{thm:fna_inf_gen}
Let $\rho$ be a strictly positive Schwartz probability density supported on $\manifold = \R^d$.
Without loss of generality, set $\kappa = 1$. Under Assumptions~\ref{ass:orthogonality} and \ref{ass:symmetry} as $t \rightarrow 0$, the limit in equation~\eqref{eq:inf_gen_FNA} exists for $f\in C_c^{\infty}(\R^d)$ and equals
\begin{equation} \label{eq:fns_infinitesimal_generator}
\mathcal{H}[f] 
= -\tilde{c}_{\alpha} \frac{(-\Delta)^{\alpha/2} (f\rho) - f (-\Delta)^{\alpha/2} \rho }{\rho}
\end{equation}
with convergence in $L_{\mathrm{loc}}^2$. The constant $\tilde{c}_{\alpha}$ is defined in equation~\eqref{eq:constants_1} (see ``Methods'').
\end{theorem} 
\smallskip

We outline the proof of Theorem \ref{thm:fna_inf_gen} in ``Methods'', drawing from the diffusion-maps framework \cite{Coifman2006,Nadler2008}.
Under Assumptions~\ref{ass:orthogonality} and
\ref{ass:symmetry}, the query--key distance reduces to
$\lVert \mathbf{x}-\mathbf{y}\rVert$, so the unnormalized
FNA affinity is symmetric and rotation-invariant.
For $\mathcal{M}=\mathbb{R}^d$, we establish the operator
limit by applying the small-bandwidth expansion in
Lemma~\ref{lem:expansion} to the numerator and denominator
of the normalized attention operator and expanding their
ratio (see ``Methods'').
This argument does not require the finite-bandwidth
FNA kernel to satisfy the heat-kernel semigroup property.

The Euclidean limit in equation~\eqref{eq:fns_infinitesimal_generator} depends on the token density $\rho$.  
To illustrate this, Appendix~\ref{supp_sub_sec:pretrained_model_embeddings} shows examples of nonuniform embedding distributions from pretrained models.
Density-normalization procedures for the Euclidean Gaussian and power-law kernels are discussed in Appendices~\ref{supp_sub_sec:diffusion_regimes} and \ref{supp_sub_sec:density_normalization}, respectively. 
Although the orthogonality and symmetry assumptions are used in the Euclidean operator limit, these conditions can be relaxed when implementing FNA, as in several of the experiments below.

\begin{figure*}[!t]
\centering
\includegraphics[scale=1]{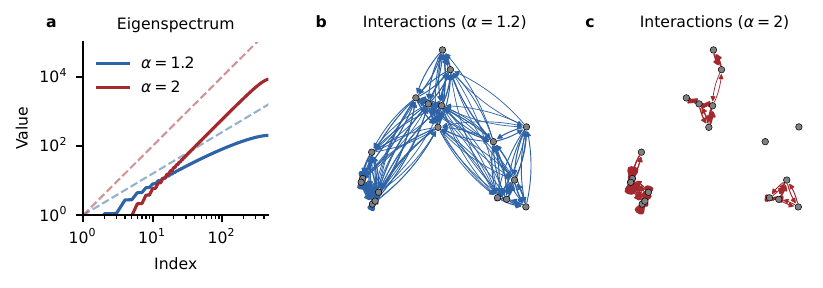} 
\caption{\label{fig:fna_token_and_spectrum} 
\textbf{FNA token interactions and spectral characteristics.}
\textbf{a} Solid lines show the estimated generator decay rates $\widehat{\lambda}_j^\varepsilon\coloneqq-\varepsilon^{-1}\log\eta_j$ (equation~\eqref{eq:spherical_decay_rates}) in increasing order, on a regular angular grid of 500 samples on the unit circle. 
We compute the decay rate from each positive attention-matrix eigenvalue $\eta_j$. The dashed lines show the reference power laws $j^{1.2}$ (blue) and $j^2$ (red), where $j$ is the ordered eigenvalue index.
\textbf{b} The grey-shaded circles represent the randomly sampled token embeddings from a tri-modal Gaussian distribution in $\R^2$.
The arrows point from queries to keys.
The thickness of the blue lines reflects the strength of the multiscale attention weights, i.e., $1/\vert \log A_{ij} \vert$ scaled by a constant.
\textbf{c} Same as in \textbf{b} but for $\alpha = 2$ with the red lines representing the local attention weight strengths.
}
\end{figure*}

We next examine the spectral properties of the spherical construction using a regular angular grid of 500 samples on $\mathbb{S}^{1}$ (Fig.~\ref{fig:fna_token_and_spectrum}\hyperref[fig:fna_token_and_spectrum]{a}). For $\alpha=2$, the corresponding continuum decay rates are those of $\tfrac14(-\Delta_{\mathbb{S}^1})$ and grow as $j^2$. For $1\leq\alpha<2$, the geodesic power-law kernel defines a nonlocal integral generator. 
Over the numerically resolved spectral range, the decay rates exhibit scaling consistent with the $j^\alpha$ growth of the spectral fractional Laplacian on the circle. 
This agreement concerns the eigenvalue growth exponent rather than an exact correspondence between the underlying operators. 
The decay rates are computed from the positive eigenvalues of the row-normalized attention matrix, as detailed in ``Methods''.
These observations motivate extending the infinitesimal-generator analysis underlying Theorem~\ref{thm:fna_inf_gen} to spherical manifolds, a direction we leave for future work.

Furthermore, we confirm that FNA promotes multiscale interactions in the $\alpha < 2$ regime. 
To illustrate this, we impose Assumptions~\ref{ass:orthogonality} and \ref{ass:symmetry} on FNA (equation~\eqref{eq:frac_attn_score}) with $\alpha = 1.2$ and $2$, yielding $\mathbf{Q} = \mathbf{K} = \mathbf{X}$.
The token embeddings (gray circles) $\mathbf{X} \in \mathbb{R}^{d \times n}$, where $(n, d) = (18, 2)$, are sampled from a two-dimensional multimodal distribution 
\begin{equation}
\rho(\mathbf{x}) = \frac{1}{3} \sum_{i=1}^3 f_2(\mathbf{x}; \bm{\upmu}_i, \bm{\Sigma})
\end{equation} 
consisting of two-dimensional Gaussian distributions $f_2(\mathbf{x}; \bm{\upmu}, \bm{\Sigma})$ with three well-separated modes centered at $\bm{\upmu}_1 = (3, 0)$, $\bm{\upmu}_2 = (-3, 0)$ and $\bm{\upmu}_3 = (0, \sqrt{3})$, respectively, and the covariance matrix $\bm{\Sigma} = I / 4$.
We observe strong multiscale connectivity, including inter- and intra-cluster interactions, for $\alpha < 2$ (Fig.~\ref{fig:fna_token_and_spectrum}\hyperref[fig:fna_token_and_spectrum]{b}).
In contrast, $\alpha=2$ exhibits intra-cluster connectivity only (Fig.~\ref{fig:fna_token_and_spectrum}\hyperref[fig:fna_token_and_spectrum]{c}). 
Because of these distinct characteristics, we henceforth refer to the cases of $\alpha < 2$ and $\alpha = 2$ as multiscale FNA and local attention, respectively.

Theorem~\ref{thm:fna_inf_gen} characterizes the limiting generator for general test functions on $\R^d$.
In a standard Transformer block, the attention update corresponds to $f(\mathbf{x}) \coloneqq \mathbf{W}_{\mathrm{V}} \mathbf{x}$ in equation~\eqref{eq:sa_resnet}. 
Under this specialization, the operator-level result of the infinitesimal generator can be further extended to the evolution of the token density through the corresponding continuity equation, i.e., the FNA version of equation~\eqref{eq:continuity_eq_1}. 
For completeness, we also derive a multi-head extension for the Euclidean case, showing that an appropriately constrained multi-head FNA formulation retains the fractional-diffusion limit in equation~\eqref{eq:frac_heat_eq}, up to a constant multiplying the fractional Laplacian. 
The full formulation, assumptions and derivation are provided in Appendix~\ref{supp_sec:multihead_fna}.

\section*{Mechanisms of FNA}

In this section, we provide a mechanistic account of the computational advantages of FNA by comparing multiscale FNA ($\alpha < 2$), local attention ($\alpha = 2$) and vanilla dot-product attention (DP). 
We adopt a minimal modeling approach, widely used in theoretical deep-learning studies \cite{Lu2019,Vuckovic2020,Bahri2020,Sander2022,Geshkovski2023b}, and focus on single-head attention with orthogonal query and key projections, comparing models with and without query--key weight tying.
We then train all models on the IMDb reviews dataset \cite{Maas2011} for binary text classification. We first show multiscale FNA achieves the best performance with fewer parameters, indicating a more expressive and powerful mechanism. Adopting a graph-theoretic view of attention interactions, we then perform an ablation study via random edge removal to confirm the superior expressivity of multiscale attention. Finally, using graph spectral analysis and diffusion maps, we elucidate how multiscale attention achieves enhanced expressivity with fewer parameters.

\subsection*{Fewer parameters.}
Multiscale FNA achieves strong performance with substantially smaller model configurations. 
To demonstrate this, we perform a hyperparameter sweep over embedding dimensions $d$ for all models under Assumption~\ref{ass:orthogonality} (orthogonality). 
At $L = 1$, multiscale FNA maintains consistently high performance across embedding dimensions $d$, whereas local attention ($\alpha = 2$) and dot-product attention (DP) are more sensitive to reduced dimensionality (Fig.~\ref{fig:imdb_ablation}\hyperref[fig:imdb_ablation]{a}). 
Notably, multiscale FNA reaches its best performance with a single attention layer, whereas local attention and DP require greater model capacity, through increased embedding dimension, to approach their saturated performance. 
A particularly striking example is a single-layer, single-head multiscale FNA model with $d = 8$, which outperforms the substantially larger six-layer, eight-head Transformer baseline despite its much smaller parameter budget. 
Together, these results demonstrate that multiscale FNA achieves greater performance per parameter and motivate the expressivity analyses presented below.
This parameter efficiency also has practical computational implications. 
The cost of a self-attention layer scales as $O(n^2 d + d^2n)$, and FNA has the same per-layer asymptotic cost as dot-product attention. 
The total cost of the model also grows linearly with the number of layers $L$. 
Achieving strong performance with a smaller embedding dimension $d$ and fewer layers $L$ can thus reduce the overall computational cost. 
This advantage is complementary to efficient-attention methods that target the dependence on sequence length $n$, either by sparsifying the attention pattern \cite{Zaheer2020,beltagy2020longformer} or by reducing memory traffic in exact attention \cite{Dao2022}.

\begin{figure}[!tbp]
    \centering
    \includegraphics[scale=0.65]{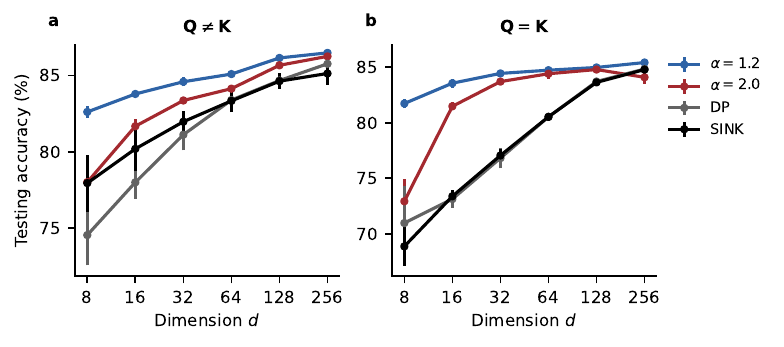}
    \caption{\textbf{Effects of embedding dimension.} 
    \textbf{a} Test accuracy as a function of embedding dimension $d$ for single-layer models ($L=1$) with $\mathbf{Q} \neq \mathbf{K}$. 
    Dots and error bars indicate the mean and standard deviation, respectively, across five trials. 
    The standard deviation is small for both multiscale ($\alpha = 1.2$) and local ($\alpha = 2$) attention, but visible for both Transformer and Sinkformer at various degrees depending on the embedding dimension $d$.
    Blue, multiscale FNA ($\alpha=1.2$); red, local attention ($\alpha=2$); grey, Transformer with dot-product attention (DP); black, Sinkformer (SINK). $\mathbf{Q}$ and $\mathbf{K}$ are query and key matrices.
    \textbf{b} Same as \textbf{a}, but with weight-tying $\mathbf{Q} = \mathbf{K}$.
    }
    \label{fig:imdb_ablation}
\end{figure}

\begin{figure*}[!t]
    \centering
    \includegraphics[scale=1]{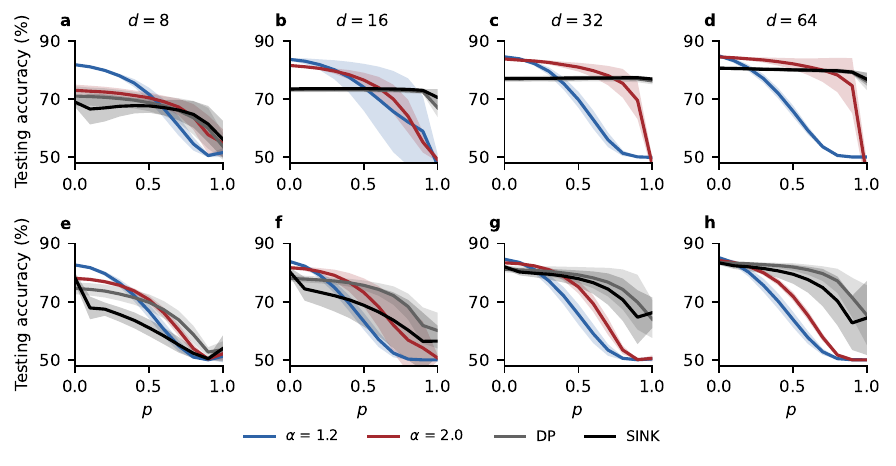}
    \caption{\textbf{Probabilistic masking of attention scores.}
    Test accuracy after independently masking each query--key score with probability $p$ in pretrained FNA models, Transformers (DP) and Sinkformers (SINK). Lines show the mean across five trials for each network with $\mathbf{Q} = \mathbf{K}$. Shaded regions indicate the standard deviation.
    \textbf{a} $d=8$. \textbf{b} $d=16$. \textbf{c} $d=32$. \textbf{d} $d=64$.
    \textbf{e}--\textbf{h} Same as \textbf{a}--\textbf{d}, but without weight-tying ($\mathbf{Q} \neq \mathbf{K}$).
    }
    \label{fig:dynamic_inference_all}
\end{figure*}

\subsection*{Greater expressivity.} 
To demonstrate how FNA’s enhanced performance, especially in the multiscale regime, arises from a more expressive attention mechanism, we perform an ablation analysis in which query--key attention scores are progressively masked from the attention graph. 
Specifically, each key is independently masked with Bernoulli probability $p$ in the attention-score matrix (see ``Methods''). 
From a graph perspective, this suppresses the attentional influence of the selected tokens by removing their incoming attention connections. 
As expected, edge-removal on the pretrained networks with weight-tying leads to decaying accuracies on the testing dataset (Fig.~\ref{fig:dynamic_inference_all}). 
However, Transformers (DP) and Sinkformers (SINK) exhibit slightly less degradation than FNA at $d=8$, with the difference becoming substantially more pronounced with $d > 8$.
At $p=1$, DP and SINK replace every score with the same logit before normalization and retain value aggregation with constant weights. Their accuracies remain above chance, particularly for $d\geq16$. In FNA, the weighted sum of value vectors is zero at $p=1$, and accuracy falls to chance level. The endpoint therefore compares different interventions across architectures.
These results indicate that DP models derive most of their expressivity from the fully-connected layers following the attention block rather than the attention layers during training. 
In contrast, the attention mechanism of FNA models is intrinsic to its expressivity: without it, performance collapses to chance-level.
Together with the parameter-efficiency results above, this supports the view that FNA concentrates a greater fraction of the model’s representational capacity within the attention operation itself. 
Qualitatively similar, although less pronounced trends are observed without query--key weight tying (Fig.~\ref{fig:dynamic_inference_all}\hyperref[fig:dynamic_inference_all]{e}--\hyperref[fig:dynamic_inference_all]{h}) and when ablation is performed systematically based on the locality of token representations (Appendix~\ref{supp_sub_sec:text_classification}).

\subsection*{Mechanistic explanation.}
We analyze FNA through the lens of an attention graph: self-attention induces a random-walk process on a directed graph whose nodes are tokens and edges encode attentional interactions. 
We assign each directed edge a path cost equal to the reciprocal of its attention weight, $1/A_{ij}$, so that stronger attention corresponds to a lower cost.
The dynamics of the random walk are governed by spectral properties of the attention matrix $\mathbf{A}$ (and $\tilde{\mathbf{A}}$). 
For the tied identity query and key projections used in Fig.~\ref{fig:fna_mechanism}\hyperref[fig:fna_mechanism]{a}, the attention matrix has a real spectrum. We define the spectral gap as $1-\eta_1$, where $\eta_1$ is the second-largest eigenvalue by value. In these reversible attention graphs, a larger gap is associated with faster mixing. We calculate the mean gap over 100 sequences sampled from the IMDb training set (see ``Methods'').
We observe a strong correlation between the spectral gap and test accuracy for both multiscale FNA and local attention, whereas this relationship is absent for DP; differences among various $\alpha < 2$ are negligible, due to the robustness of multiscale attention explained above. 
Prior work has also applied the link between spectral gap and mixing time, although through a heuristic argument, to motivate their sparse attention formulation \cite{Zaheer2020}.
In contrast, our results reveal that FNA, unlike vanilla attention, suggests the spectral gap as a predictive index for quantifying performance, due to its interpretable diffusion formulation. 

\begin{figure*}[!t]
    \centering
    \includegraphics[scale=0.9]{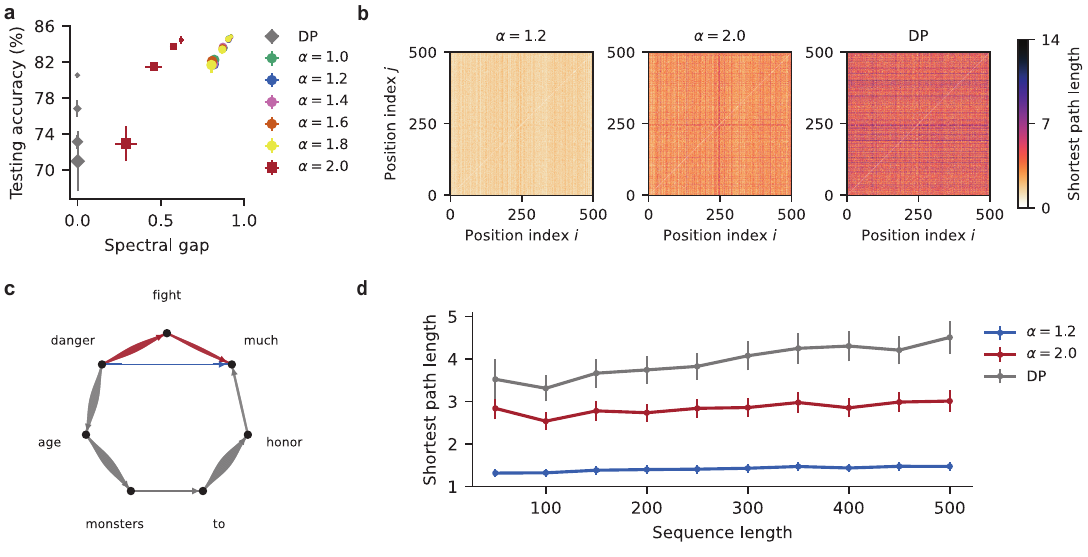}
    \caption{\textbf{Graph-theoretic analysis of fractional neural attention.} 
    \textbf{a} Mean test accuracy across five trials versus the mean spectral gap of attention weights computed on a random IMDb subset of 100 random sequences for pretrained single-layer FNA and Transformer (DP) models. 
    Error bars show the standard deviation.
    The results for $\alpha < 2$ reveal a high degree of overlap.
    \textbf{b} Matrices of shortest-path lengths between tokens in a single sequence for multiscale FNA ($\alpha=1.2$), local attention ($\alpha=2$) and DP for embedding dimension $d=8$. The colour scale shows shortest-path length in graph steps.
    \textbf{c} Attention flow between two related tokens indicated by arrows: ``much'' and ``danger'', with line width reflecting strength. 
    The shortest paths for $\alpha=1.2$, $\alpha=2$, and DP are 1, 2 and 5, respectively. 
    \textbf{d} Dots show the mean shortest path length between each pair of token in a random sequence of the specific length. Error bars represent 0.2 standard deviations (scaled for clarity). 
    }
    \label{fig:fna_mechanism}
\end{figure*}

We further show that this enhanced connectivity of multiscale FNA manifests as shorter path lengths on the attention graph. 
We define shortest-path length as the number of edges in a minimum-cost directed path. This statistic characterizes effective token connectivity and complements our empirical finding that multiscale FNA achieves strong performance with smaller model sizes.
As an illustrative example, derived from the same 500-token sequence in Fig.~\ref{fig:fna_mechanism}\hyperref[fig:fna_mechanism]{b}, we observe that the minimum-cost route from ``much'' to ``danger'' is direct for $\alpha=1.2$ and spans $2$ and $5$ edges for $\alpha=2$ and DP, respectively (Fig.~\ref{fig:fna_mechanism}\hyperref[fig:fna_mechanism]{c}).
The matrix of shortest path lengths between any two tokens reveals that this is general across the entire sequence (Fig.~\ref{fig:fna_mechanism}\hyperref[fig:fna_mechanism]{b}); the maximum shortest path lengths are 4, 11 and 14 for $\alpha = 1.2$, $\alpha = 2$ and DP, respectively.
By repeating this analysis, we determine that these differences in shortest path lengths are robust across the sequence lengths tested (Fig.~\ref{fig:fna_mechanism}\hyperref[fig:fna_mechanism]{d}). 
The mean number of edges per minimum-cost path is just over one for multiscale FNA ($\alpha=1.2$). For local attention and DP, it is larger and grows slowly with sequence length.
These results indicate that FNA is effective for long sequences and explain why multiscale attention can achieve optimal performance with only a single layer and small embedding dimensions, while local attention and DP typically need large embedding dimensions or multiple layers for performance to saturate (Fig.~\ref{fig:imdb_ablation}).

\begin{figure*}[!t]
    \centering
    \includegraphics[scale=0.8]{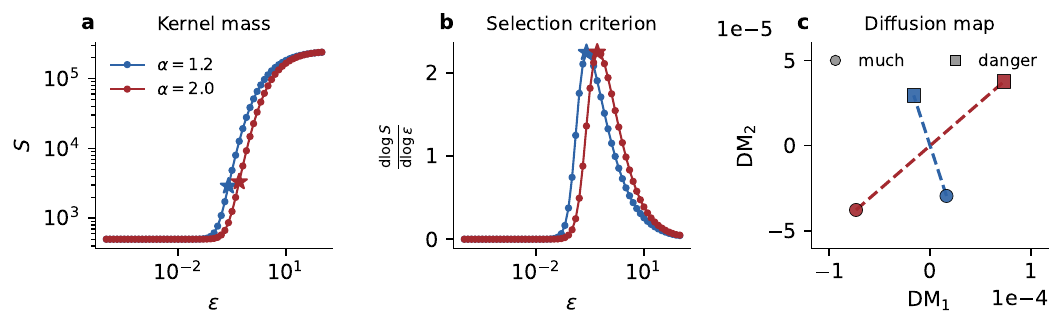}
    \caption{\textbf{Diffusion maps of FNA.} 
    \textbf{a} Summation of FNA scores $S(\varepsilon)$ vs the bandwidth parameter $\varepsilon$ for $\alpha = 1.2$ (blue) and $\alpha = 2$ (red) on a log-log plot.
    \textbf{b} Corresponding to \textbf{a}, the derivative $\frac{\d \log S}{\d \log \varepsilon}$ plotted against the bandwidth parameter $\varepsilon$.
    \textbf{c} The first two centered diffusion-map coordinates for the tokens ``much'' and ``danger''. The uncentered coordinates are $\mathrm{DM}_i \coloneqq \eta_i^{\tau} \bm{\uppsi}_i(\tilde{\mathbf{x}})$.
    Circles and squares correspond to the tokens ``much'' and ``danger'', respectively. 
    The dashed lines approximate the diffusion distance in the first two diffusion-map coordinates ($m=2$).
    }
    \label{fig:fna_dm}
\end{figure*}

\subsection*{Dimensionality reduction and interpretability of FNA.}
The construction of FNA allows us to further understand the connectivity of tokens within the graph representation by leveraging tools from geometric harmonics.
In particular, we perform a dimensionality reduction on linearly scaled token embeddings $\tilde{\mathbf{x}} \coloneqq \mathbf{x} / \kappa$ using a diffusion map $\bm{\Psi}_{\tau,\varepsilon} : \mathbb{R}^d \rightarrow \mathbb{R}^m$, whose value at token index $i$ is
\begin{equation}\label{eq:diffusion_map}
\bm{\Psi}_{\tau,\varepsilon}(\tilde{\mathbf{x}}_i)
=\left(\eta_1^\tau\uppsi_1(i),\ldots,\eta_m^\tau\uppsi_m(i)\right),
\end{equation}
where $\tau$ is the number of graph-walk steps, $m$ is the number of retained nonstationary modes and $\bm{\uppsi}_k$ is the right eigenvector paired with eigenvalue $\eta_k$. With the left eigenvectors $\bm{\upphi}_k$ normalized as in ``Methods'',
\begin{equation}\label{eq:attn_weight_decomposition}
[(\mathbf A_\varepsilon)^\tau]_{ij}
=\sum_{k=0}^{n-1}\eta_k^\tau\uppsi_k(i)\upphi_k(j).
\end{equation}
The bandwidth $\varepsilon$ controls the spatial scale of the affinities. Diffusion distance compares the two $\tau$-step transition distributions, weighted by inverse stationary probability. Euclidean distance in the diffusion map equals it for $m=n-1$ and gives a lower bound when only the first $m$ modes are retained \cite{Coifman2006,Nadler2008}.
This measures the similarity between the data points based on their connectivity within a graph structure.
It quantifies how easily one can transition from one data point to another assuming a random walk process on the graph, determined by its infinitesimal generator and thus the corresponding construction of attention in equation~\eqref{eq:inf_gen_FNA}.
Here, the process is differentiated by $\alpha$ with L\'evy process ($\alpha < 2$) and Brownian motion ($\alpha = 2$).
Points in proximity with respect to the diffusion distance are more likely to be connected by paths of a certain length, reflecting better information flow from one token to another; the larger the distance, the more restricted is this flow. 

We select the bandwidth $\varepsilon$ using the $\log$-$\log$ criterion \cite{Singer2009}. 
This entails determining $\varepsilon$ from the scaling behavior of the unnormalized kernel by defining $S(\varepsilon)$ as the sum of its attention-score entries (equation~\eqref{eq:att_score_sum}). 
We then examine $S(\varepsilon)$ over a range of bandwidths and choose the value of $\varepsilon$ at which the logarithmic derivative $\frac{\d \log S}{\d \log \varepsilon}$ is maximal (see ``Methods''). 
For the randomly selected sequence of length $n = 500$, $S(\varepsilon)$ and its logarithmic derivative are shown for multiscale FNA ($\alpha = 1.2$) and local attention ($\alpha = 2$) in Fig.~\ref{fig:fna_dm}\hyperref[fig:fna_dm]{a}--\hyperref[fig:fna_dm]{b}, respectively. 
This procedure provides a data-dependent choice of bandwidth for the subsequent diffusion-map analysis.

Using the selected bandwidth, we construct a two-dimensional diffusion-map representation for the same sequence analyzed in the attention-graph analysis above (Fig.~\ref{fig:fna_dm}\hyperref[fig:fna_dm]{c}). 
The resulting coordinates reveal marked differences between local and multiscale attention. 
Under the diffusion map embeddings with $m = 2$ in equation~\eqref{eq:diffusion_map}, the words ``much'' and ``danger'', whose minimum-cost route consists of one edge for $\alpha=1.2$ and two edges for local attention, are close in diffusion space for multiscale FNA and substantially farther apart for $\alpha = 2$, where the dashed lines approximate the diffusion distance (equation~\eqref{eq:diffusion_dist}--\eqref{eq:diffusion_dist_approx}).
The graph analysis independently shows shortest-path lengths of 1 and 2 for these two regimes, respectively. 
Their reduced diffusion distance under multiscale FNA thus reflects stronger graph-mediated connectivity rather than merely greater proximity in the original embedding space.

This geometric view complements the spectral-gap and shortest-path analyses: the spectral gap characterizes global mixing, shortest-path lengths count the edges along minimum-cost routes, and diffusion distance captures the similarity of their multi-step transition neighborhoods. 
Together, these measures provide an interpretable account of how multiscale FNA reorganizes the attention graph to facilitate information flow across distant token representations.

\section*{Experiments}

We evaluate FNA against standard Transformers and Sinkformers \cite{Sander2022} across various tasks and modalities, including sentiment analysis, image classification with Vision Transformers (ViTs) \cite{Dosovitskiy2020}, and neural machine translation, which additionally requires an autoregressive decoder. 
Model architectures and training procedures are described in ``Methods''. 

\begin{figure}[!htbp]
    \centering
    \includegraphics[scale=0.65]{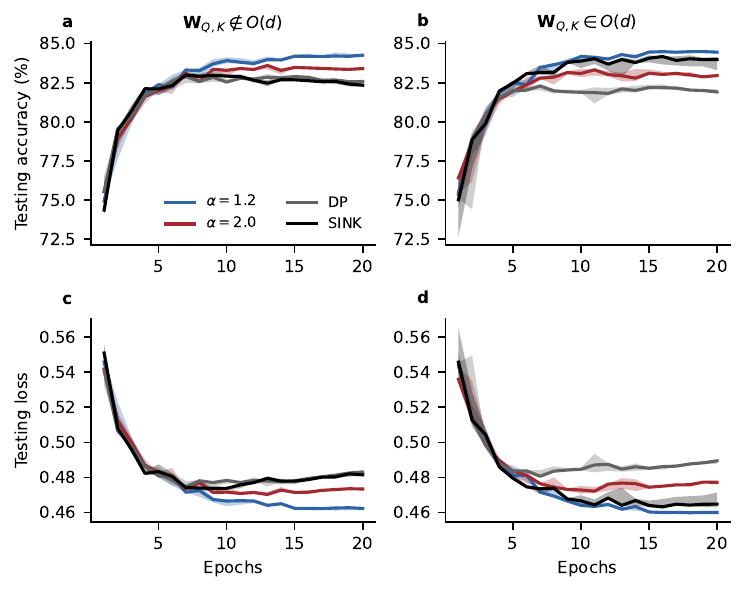}
    \caption{\textbf{Learning curves for text classification on IMDb.} \textbf{a} Median test accuracy across five trials for multiscale FNA ($\alpha = 1.2$), local attention ($\alpha = 2$), Transformer (DP) and Sinkformer (SINK) using $\mathbf{W}_{\mathrm{Q,K}} \notin O(d)$. 
    Shading spans the 25th to 75th percentiles across the five trials at each epoch. 
    Matrices $\mathbf{W}_{\mathrm{Q}}, \mathbf{W}_{\mathrm{K}}$ (written as $\mathbf{W}_{\mathrm{Q},\mathrm{K}}$) denote the query and key projection matrices, and $O(d)$ is the orthogonal group in dimension $d$.
    \textbf{b} Same as \textbf{a}, but with $\mathbf{W}_{\mathrm{Q,K}} \in O(d)$. 
    \textbf{c} Same as \textbf{a}, but showing test loss. 
    \textbf{d} Same as \textbf{c}, but with $\mathbf{W}_{\mathrm{Q,K}} \in O(d)$.
    \label{fig:train_imdb}}
\end{figure}

\begin{table}[!htbp]
\centering
\begin{tabular}{l|c|c|c|c}
\specialrule{.1em}{.05em}{.05em}  
Model & Mean & Median & Best & Worst \\
\hline
FNA & 84.26\% & 84.26\% & 84.37\%  & 84.18\% \\
LA & 83.37\% & 83.40\% & 83.48\% & 83.24\% \\
SINK & 82.44\% & 82.34\% & 82.79\% & 82.24\% \\
DP & 82.55\% & 82.58\% & 82.68\% & 82.39\% \\
\hline 
op-FNA & \textbf{84.45\%} & \textbf{84.45\%} & \textbf{84.68\%} & \textbf{84.27\%} \\
op-LA & 82.95\% & 82.96\% & 83.16\% & 82.71\% \\
op-SINK & 83.63\% &  83.98\% & 84.29\% & 82.40\% \\
op-DP & 81.97\% & 81.91\% & 82.22\% & 81.76\% \\
\specialrule{.1em}{.05em}{.05em} 
\end{tabular}
\vspace{0.2cm}
\caption{\label{table:train_imdb} 
Test accuracy on IMDb dataset acquired over 5 independent iterations.
Bold indicates the highest accuracy in each column across the model configurations. 
}
\end{table}

\subsection*{Text classification.}
We test model performance on the IMDb dataset using a more realistic model configuration of 6 layers and 8 attention heads.
For $\mathbf{W}_{\mathrm{Q,K}} \notin O(d)$, FNA model with $\alpha = 1.2$ achieves the highest median evaluation accuracy by the end of training for 20 epochs across 5 trials (Fig.~\ref{fig:train_imdb}\hyperref[fig:train_imdb]{a} blue curve).
The Transformer (DP) and Sinkformer (SINK) reach approximately $83\%$ at epoch 7 and plateau thereafter, but the Sinkformer's accuracy begins to decline slightly after epochs 8--10.
By contrast, local attention (red curve) and multiscale FNA (blue curve) models continue to improve (Fig.~\ref{fig:train_imdb}\hyperref[fig:train_imdb]{a}).
When $\mathbf{W}_{\mathrm{Q,K}} \in O(d)$, we refer to this as the orthogonal-projection (op) setting for the rest of the text. 
Under this setting, the overall behavior remains qualitatively similar; however, the test accuracy gap between op-FNA $(\alpha = 1.2)$ and op-DP widens further, as shown in Fig.~\ref{fig:train_imdb}\hyperref[fig:train_imdb]{b}. 
Moreover, op-SINK exhibits improved performance, outperforming op-LA while still falling short of op-FNA.
The loss curves for all models reflect their corresponding accuracies (Fig.~\ref{fig:train_imdb}\hyperref[fig:train_imdb]{c}--\hyperref[fig:train_imdb]{d}), showing an overall decreasing trend.
However, FNA (blue curves) exhibits the smallest degradation for both cases of $\mathbf{W}_{\mathrm{Q,K}} \notin O(d)$ and $\mathbf{W}_{\mathrm{Q,K}} \in O(d)$.
The overall accuracies of the 6-layer models can already be closely matched by the single-layer models or overtaken in Fig.~\ref{fig:imdb_ablation}, with multiscale FNA exhibiting superior trainability.
This is potentially attributed to the simplicity of the IMDb dataset, as larger models can easily overfit the training dataset.
More complex learning rate schedulers can overcome this issue through introducing warm-up, however, we refrain from doing this as a means to compare the trainability of architectures.
We summarize all model performances on the IMDb dataset in Table~\ref{table:train_imdb}.

To further evaluate multiscale FNA and assess the sensitivity of its performance to changes in task characteristics, we train the models on AG News \cite{Zhang2015}, which contains four classes and shorter sequences than IMDb and thus provides a different data structure setting.
Under the experimental conditions considered, multiscale FNA achieves the highest performance among the evaluated models, 
consistent with our findings on IMDb.
Similarly, across reduced training-set sizes obtained by repartitioning the 50,000 labeled IMDb reviews, multiscale attention generally outperforms local-attention models, Transformers, and Sinkformers, 
although the margin in performance narrows with increasing embedding dimension $d$. 
Multiscale attention also yields the strongest results among the evaluated methods as the context-window length \(n\) is increased from \(128\) to \(1024\). 
Moreover, for the 6-layer, 8-head models, similar performance is maintained across all tested values of \(\alpha<2\). 
All these additional findings can be found in Appendix~\ref{supp_sub_sec:text_classification}.

\subsection*{Image processing.} \label{sec:image_processing}
Following standard practice, we treat an image as a sequence by feeding ordered, fixed-size patches into the attention mechanism, yielding a ViT architecture \cite{Dosovitskiy2020}. 
We train FNA-based models and compare them with the Transformer and Sinkformer \cite{Sander2022} baselines on CIFAR-10 \cite{cifar10}, and also examine the effect of enforcing Assumption~\ref{ass:orthogonality} (orthogonality).

\begin{figure}[!tbp]
\centering
\includegraphics[scale=0.65]{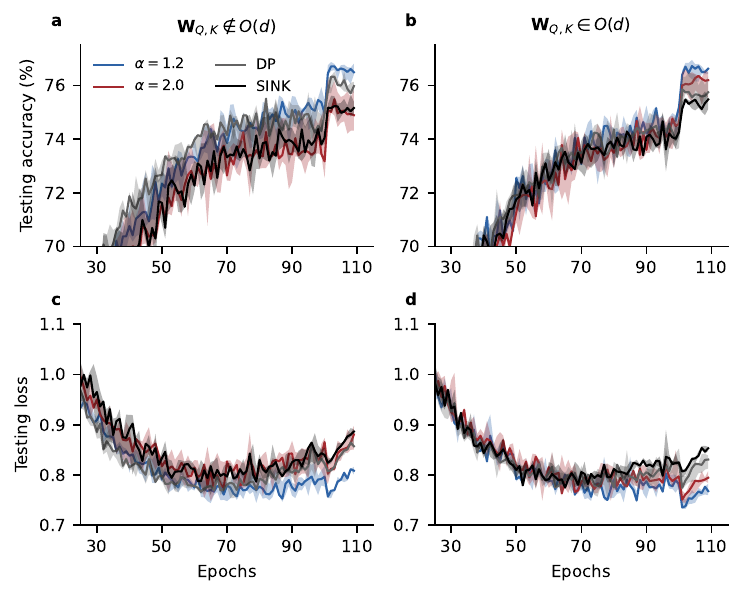}
\caption{ 
\textbf{CIFAR-10 classification.}
\textbf{a} Median test accuracy across five trials for FNA, local attention, Transformer (DP) and Sinkformer (SINK), each with four layers, six attention heads and $\mathbf{W}_{\mathrm{Q,K}}\notin O(d)$.
Shading spans the 25th to 75th percentiles across the five trials at each epoch.
\textbf{b} Same as \textbf{a}, but with $\mathbf{W}_{\mathrm{Q,K}} \in O(d)$.
\textbf{c} Same as \textbf{a}, but showing test loss.
\textbf{d} Same as \textbf{c}, but with $\mathbf{W}_{\mathrm{Q,K}} \in O(d)$.
\label{fig:train_cifar10}
}
\end{figure}
 
At the end of training, multiscale FNA ($\alpha = 1.2$) achieves a median accuracy of 76.49\%, 
surpassing its local counterpart (74.89\%) and outperforming both the Transformer baseline (75.98\%) and Sinkformer (75.16\%) (Fig.~\ref{fig:train_cifar10}\hyperref[fig:train_cifar10]{a}).
When orthogonality is enforced ($\mathbf{W}_{\mathrm{Q,K}} \in O(d)$), multiscale FNA exhibits median accuracy nearly unchanged at 76.62\%. 
Local attention and the Sinkformer improve to 76.19\% and 75.48\%, respectively, while the Transformer degrades slightly to 75.75\% (Fig.~\ref{fig:train_cifar10}\hyperref[fig:train_cifar10]{b}). 
The test loss exhibits the corresponding inverse trends in Fig.~\ref{fig:train_cifar10}\hyperref[fig:train_cifar10]{c}--\hyperref[fig:train_cifar10]{d}.

\subsection*{Machine translation.}
We next evaluate FNA on a translation task that requires a decoder component in addition to the encoder, making the problem substantially more challenging. 
We train all models on the Multi30K English-to-German (En-De) dataset \cite{Elliott2016}, using 40 training epochs and 5 independent runs per configuration. 
Multiscale FNA achieves higher BLEU (bilingual evaluation understudy) scores \cite{Papineni2002} than local attention, DP and SINK models; summary results are reported in Table~\ref{table:train_translate}, and training curves shown in Appendix~\ref{supp_sub_sec:translation}.  

\begin{table}[!htbp]
\centering
\begin{tabular}{l|c|c}
\specialrule{.1em}{.05em}{.05em}  
Model & $\mathbf{W}_{\mathrm{Q, K}} \notin O(d)$ & $\mathbf{W}_{\mathrm{Q, K}} \in O(d)$ \\
\hline
FNA & \textbf{36.17} $\pm$ \textbf{0.54} & \textbf{34.70} $\pm$ \textbf{0.33} \\
LA  & $35.77 \pm 0.21$ &  $34.05 \pm 0.39$ \\
DP  & $35.66 \pm 0.35$ & $34.61 \pm 0.36$ \\
SINK & $35.92 \pm 0.14$ & $34.37 \pm 0.42$ \\
\specialrule{.1em}{.05em}{.05em} 
\end{tabular}
\vspace{0.2cm}
\caption{\label{table:train_translate} 
The mean BLEU score and its standard deviation across 5 iterations for each model. 
Bold indicates the highest mean BLEU score in each column. 
}
\end{table}

Across the experiments above, we use $\alpha = 1.2$ as a representative multiscale setting, given that performance is largely insensitive to the choice of $\alpha < 2$ (Fig.~\ref{fig:fna_mechanism}\hyperref[fig:fna_mechanism]{a} and Appendix~\ref{supp_sub_sec:text_classification})).  
We have also examined the effects of imposing Assumption~\ref{ass:orthogonality} (orthogonality) and, for single-layer models, Assumption~\ref{ass:symmetry} (weight-tying). These constraints affect models and tasks differently. 
For example, orthogonality improves local attention and the Sinkformer on CIFAR-10 but lowers BLEU for all models on Multi30K. 
Nevertheless, multiscale FNA attains the best performance in all settings considered here, although in several cases the margin is within one standard deviation.
At the same time, the performance gap narrows as model capacity and dataset size increase, suggesting the possibility of a crossover at larger scales. Establishing whether such a crossover occurs, however, would require a dedicated scaling study involving substantially larger models and datasets, which lies beyond the scope of the present work.

\section*{Discussion}

In this study, we have introduced fractional neural attention (FNA), a neuroscience-inspired attention framework that unifies principles of attentional sampling with dynamical systems theory, network science, and Transformer architectures. 
FNA naturally captures token interactions across multiple scales, yielding greater expressivity and faster information mixing, and thus advances the foundational capacity of Transformers. 
Theoretically, we have illustrated that the dynamics of FNA are governed by a partial differential equation (PDE) mechanistically linked to the fractional Laplacian, and that larger spectral gaps and shorter path lengths in the induced attention networks underlie its computational advantages. 
Empirically, FNA matches or exceeds Transformer baselines including Sinkformers across text, vision, and machine translation tasks, suggesting substantial headroom for scaling to larger models.  

A growing line of work has approached attention mechanisms through the lenses of stochastic and dynamical systems, treating token embeddings as spatial variables and depth as a time component \cite{Lu2019,Vuckovic2020,Vuckovic2021,Geshkovski2023b}. 
For instance, the transformation of token embeddings in a sequence as they pass through Transformer layers has been interpreted as the evolution of a multi-particle system using numerical methods like the Lie-Trotter splitting scheme \cite{Lu2019}.
Bridging deep neural networks with ODEs \cite{Chen2018} has also been adopted in other studies to understand and describe the geometry of learned representations of Transformers with time-independent weights \cite{Geshkovski2023b}.
In a similar vein, this framework enables the application of measure-theoretic tools to investigate the regularity properties of Transformers \cite{Vuckovic2020,Vuckovic2021}.
Additionally, under suitable normalization---e.g., through the Sinkhorn algorithm to enforce bistochasticity on attention scores and additional conditions on the projection matrices (see equation~\eqref{eq:heat_eq}), the resulting attention map converges to a local Laplace operator, yielding heat-equation dynamics (Brownian motion) for self-attention \cite{Sander2022}. 
When imposing further geometrical structure such as weight-tying and orthogonality, the isotropic fractional heat equation can be further derived, directly linking attention to $\alpha$-stable L\'evy dynamics with multiscale, power-law interactions rather than Brownian diffusion \cite{Applebaum_2009}. 
Under these conditions, diffusion maps also provide visual and quantitative measures of token similarity. 
These measures account for both the token distribution and the geometry of the space in which the tokens reside, with the bandwidth $\varepsilon$ selected by a data-driven criterion \cite{Singer2009}. 
Separately, by expressing the distance scale $\kappa$ in terms of the head dimension $d_h$, we identify distinct regimes of multiscale attention. 
How $\varepsilon$ and $\kappa$ should scale with the sequence length $n$ (or context window), a question recently studied for standard attention \cite{Giuseppe2026}, lies beyond the scope of this work and is left for future investigation.

FNA enables multiscale coupling of short- and long-range interactions simultaneously, because its update rule is driven by a fractional operator. 
This differs fundamentally from engineered local-global hybrids that hard-code two scales. 
For instance, Longformer combines sliding-window attention with a small set of global tokens to capture local and document-level dependencies \cite{beltagy2020longformer}; BigBird introduces block-sparse patterns with global, random, and local links to ensure theoretical completeness and scalability \cite{Zaheer2020}; and Nystr\"omformer approximates full attention with landmark-based low-rank kernels \cite{Xiong2021}. 
These designs reliably extend reach and improve efficiency, but their local-global behavior arises from hand-crafted patterns (such as window sizes, dilation rates, landmark choices) rather than from a single governing operator. 
In FNA, fractional diffusion is built into the attention rule, providing a principled mechanism for multiscale coupling. 
A single parameter $\alpha$ determines the heaviness of the interaction tail.
For $\alpha < 2$, the operator induces scale-free, power-law coupling across all distances, so multiscale behavior emerges intrinsically rather than through manually designed connectivity.
When $\alpha = 2$, the operator reduces to the local attention case, which partially overlaps with a previous attention formulation that ensures Lipschitz continuity \cite{Kim2021}.
This connection suggests that the multiscale FNA map retains desirable regularity properties, which can stabilize training \cite{Miyato2018} and facilitate the construction of flow-based generative models \cite{Behrmann2019,Chen2019}.

Related non-local neural networks capture global dependencies by aggregating information across all positions using learned weights \cite{Wang_2018_CVPR}, much like the self-attention mechanism in Transformers.
FNA shares this global aggregation structure but is distinguished by a distance-dependent interaction kernel linked to the fractional Laplacian, a non-local operator in the fractional regime.
This connection provides an interpretable geometric--stochastic foundation for understanding how inter-token influence decays with distance and how information propagates through the network.
Analyzing FNA through an attention-network perspective reveals that self-attention defines a discrete random walk analogous to a L\'evy process whose mixing speed is governed by the spectral gap of the attention matrix. 
Empirically, FNA exhibits a strong correlation between model performance and spectral gap, indicating faster information mixing while maintaining robustness across different fractional orders $\alpha < 2$.
This connection between mixing speed and model efficiency parallels the motivation behind BigBird \cite{Zaheer2020}, where sparse attention patterns approximate Erd\H{o}s-R{\'e}nyi graphs with large spectral gaps and rapid mixing.
However, FNA achieves efficiency without imposing sparsity-inducing inductive biases: it delivers substantial parameter reductions (fewer layers and lower embedding dimensions) while concentrating a greater portion of information processing within the attention mechanism itself, thus aligning the model architecture more closely with the data geometry than the original Transformer.

It is important to note that natural data, from images \cite{Munn2018,Ruderman1994,Field87} to language sequences \cite{Alabdulmohsin2024}, often exhibit fractal structure, nested multiscales and heavy-tailed statistics. 
FNA is intrinsically matched to these properties by introducing self-similar, power-law interactions characteristic of L\'evy processes, thus enabling efficient information mixing across scales. 
More broadly, nested multiscale organization is a hallmark of diverse cognitive functions \cite{Gilden1995,Kello2010,He2014}; FNA thus offers a unifying principle linking biological and artificial intelligence, and could foster deeper, synergistic interactions between neuroscience and modern AI \cite{Hassabis2017,Zador2023,Muller2024}.

\section*{Methods}
\subsection*{Riemannian manifold}
Throughout this paper, we consider only Riemannian manifolds $\manifold$.
Provided the local coordinates, its volume form is $\d \vol_{\mathbf{x}} \coloneqq \sqrt{|g|} \d x^1 \wedge \cdots \wedge \d x^n$, where $\vert g \vert$ is the absolute value of the determinant of the matrix representation of the metric tensor on $\manifold$.
Additionally, we use $d_g: \manifold \times \manifold \rightarrow \R^+$ to represent the geodesic metric. That is, for $\mathbf{x}, \mathbf{y} \in \manifold$,
\begin{equation} \label{eq:g_dist}
d_g(\mathbf{x}, \mathbf{y})=\inf_{\gamma} \int_0^1 \sqrt{g_{\gamma(t)}( \dot\gamma(t), \dot\gamma(t))} \d t.
\end{equation}
is the infimum over piecewise differentiable paths $\gamma: [0,1] \rightarrow \manifold$ between $\gamma(0) = \mathbf{x}$ and $\gamma(1) = \mathbf{y}$.

\subsection*{Fractional Laplacian on manifolds}
On a closed Riemannian manifold $\manifold$, the fractional Laplacian is defined spectrally by
\begin{equation} \label{eq:frac_laplacian}
(-\Deltam)^{\alpha/2} u(\mathbf{x}) \coloneqq \sum_{i=1}^{\infty} \lambda_i^{\alpha/2} \left\langle u, \phi_i\right\rangle_{L^2(\manifold)} \phi_i(\mathbf{x})
\end{equation}
for $u$ in the operator domain in $L^2(\manifold)$.
By compactness, the eigenfunctions $\{\phi_i\}_{i=1}^{\infty}$ of $-\Deltam$ form an orthonormal basis, with $-\Deltam \phi_i = \lambda_i \phi_i$. If $\manifold$ is connected, the eigenvalues satisfy $0 = \lambda_1 < \lambda_2 \leq \lambda_3 \leq \cdots$ and $\lim_{i \rightarrow \infty} \lambda_i = \infty$.

\subsection*{Heat kernel}
A function $k: [0, \infty) \times \manifold \times \manifold \rightarrow [0,\infty)$ is a heat kernel if for almost every $\mathbf{x}, \mathbf{y} \in \manifold$ and $\forall s, t \geq 0$:
\begin{enumerate}
    \item Positivity: 
    \begin{align}  \label{eq:kernel_positivity}
        k_t(\mathbf{x}, \mathbf{y}) \geq 0.
    \end{align}
    \item Total mass inequality: 
    \begin{align} \label{eq:kernel_total_mass}
        \int_{\mathbf{y} \in \manifold} k_t(\mathbf{x},\mathbf{y}) \d\vol_y \leq 1.
    \end{align}
    \item Symmetry:
    \begin{align} \label{eq:kernel_symmetry}
        k_t(\mathbf{x},\mathbf{y}) = k_t(\mathbf{y},\mathbf{x}).
    \end{align}
    \item Semigroup: 
    \begin{align} \label{eq:kernel_semigroup}
        k_{s+t}(\mathbf{x},\mathbf{y}) = \int_{z \in \manifold} k_s(\mathbf{x},\mathbf{z})k_t(\mathbf{z},\mathbf{y}) \d\vol_{\mathbf{z}}.
    \end{align}
    \item Approximation of identity:
    \begin{equation}
\begin{aligned} \label{eq:kernel_identity}
    &\lim _{t \rightarrow 0^{+}}\left\|\int_{y \in \mathcal{M}} k_t(\mathbf{x}, \mathbf{y}) f(\mathbf{y}) \d\vol_{\mathbf{y}} -f(\mathbf{x})\right\|_{L^2(\mathcal{M},g)}=0, \\
& f \in L^2(\manifold, g).
    \end{aligned}
\end{equation}
\end{enumerate}

Every heat kernel gives rise to an associated semigroup
\begin{equation} \label{eq:semigroup_def}
\mathcal{K}_t f(\mathbf{x}) = \int_{\mathbf{y} \in \manifold} k_t(\mathbf{x},\mathbf{y}) f(\mathbf{y}) \d\vol_{\mathbf{y}}.    
\end{equation}
A semigroup also gives rise to a quadratic form
\begin{equation} \label{eq:quadratic_form}
\xi(f) = \lim _{t \rightarrow 0^{+}}\left\langle\frac{f-\mathcal{K}_t f}{t}, f\right\rangle_{L^2(\manifold, g)}.  
\end{equation}
We focus on stochastically complete heat kernels, which have unit mass. Equivalently, their semigroups satisfy $\mathcal{K}_t 1 = 1$ for all $t > 0$ (equation~\eqref{eq:semigroup_def}). 

A quadratic form is regular if there exists a set $\mathcal{C}$ of continuous functions with compact support $\left(C_c(\mathcal{M})\right)$ that are also in the domain $\mathcal{D}(\xi)$ of $\xi$, i.e., $\mathcal{C} \subset C_c(\mathcal{M}) \cap \mathcal{D}(\xi)$ such that $\mathcal{C}$ is dense both in $C_c(\mathcal{M})$ and $\mathcal{D}(\xi)$ under appropriate norms \cite{Antil2021}.
The heat kernel dichotomy \cite{Grigor2008} holds under two conditions:
\begin{enumerate}
    \item All balls defined with metric $d_g$ in $\manifold$ are relatively compact;
    \item The heat kernel $k_t(\mathbf{x},\mathbf{y})$ is a stochastically complete and $\alpha$-scale invariant heat kernel (equation~\eqref{eq:beta_scale_kernel}) such that the associated quadratic form is regular.
\end{enumerate} 

\subsection*{L\'evy process}
A L\'evy $\alpha$-stable distribution (often termed stable distribution) \cite{Nolan2020}, is defined by a characteristic function involving a tuple $(\alpha, \beta, \sigma, \mu)$ containing the stable, skewness, scale and location parameters respectively,
\begin{equation} \label{eq:stable_dist_char_func}
\varphi(u;\alpha,\beta,\sigma,\mu) = e^{- \vert \sigma u \vert^{\alpha}(1 - i \beta \text{sgn}(u)\zeta(u; \alpha)) 
+ iu\mu} 
\end{equation}
where $\text{sgn}(u)$ is the sign of $u$ and
\begin{equation}
\zeta(u; \alpha) = 
\begin{cases}
\tan\left( \frac{\pi \alpha}{2} \right) \quad &\alpha \neq 1 \\
-\frac{2}{\pi}\log \vert u \vert &\alpha = 1
\end{cases}.
\end{equation}
Let us consider a random variable $X$ drawn from a stable distribution, i.e., $X \sim S_{\alpha}(\beta,\sigma,\mu)$.
If it is symmetric and centered with $\beta = \mu = 0$ (i.e., $X \sim S_{\alpha}(\sigma)$), it is referred to as $S\alpha S$ (symmetric $\alpha$ stable) \cite{Samorodnitsky_Taqqu}.
For $\alpha < 2$, the random variable associated with the stable distribution has infinite absolute $m^{\text{th}}$-moment for $m \geq \alpha$.

A ($S\alpha S$) L\'evy process $L_t$ is a stochastic process with:
\begin{enumerate}
    \item $L_0 = 0$.
    \item $L_t$ has independent increments.
    \item $(L_t - L_s) \sim S_{\alpha}((t - s)^{1/\alpha})$ for all $0 \leq s \leq t$.
    \item $L_t$ is stochastically continuous.    
    \item $L_t$ is RCLL (right continuous with left limits) almost surely.
\end{enumerate}

\subsection*{Orthogonal projection for single-head attention}
We now present the proof of Proposition~\ref{prop:distance_preserving}.
\begin{proof}

In $\R^d$, the distance between each query--key pair under Assumption~\ref{ass:orthogonality} satisfies
\begin{equation}
\begin{aligned}
\Vert \mathbf{q}_i - \mathbf{k}_j \Vert
&= ( \Vert \mathbf{q}_i \Vert^2 - 2 \mathbf{q}_i^{\top} \mathbf{k}_j + \Vert \mathbf{k}_j \Vert^2 )^{1/2}, \\
&= ( \Vert \mathbf{W}_{\mathrm{Q}} \mathbf{x}_i \Vert^2 - 2 \mathbf{x}_i^{\top} \mathbf{W}_{\mathrm{Q}}^{\top} \mathbf{W}_{\mathrm{K}} \mathbf{x}_j + \Vert \mathbf{W}_{\mathrm{K}} \mathbf{x}_j \Vert^2 )^{1/2}, \\
&= (\Vert \mathbf{x}_i \Vert^2 - 2 \mathbf{x}_i^{\top} \tilde{\mathbf{W}}_{\mathrm{K}} \mathbf{x}_j + \Vert \mathbf{x}_j \Vert^2)^{1/2}, \\
&= \Vert \mathbf{x}_i - \tilde{\mathbf{W}}_{\mathrm{K}} \mathbf{x}_j \Vert, \\
&= \Vert \tilde{\mathbf{q}}_i - \tilde{\mathbf{k}}_j \Vert,
\end{aligned}
\end{equation}

where $\tilde{\mathbf{W}}_{\mathrm{K}} \coloneqq \mathbf{W}_{\mathrm{Q}}^{\top} \mathbf{W}_{\mathrm{K}} \in O(d)$, allowing us to set without loss of generality $\mathbf{Q} = \mathbf{X}$ and $\mathbf{K} \coloneqq \tilde{\mathbf{W}}_{\mathrm{K}} \mathbf{X}$. 
For $\manifold = \sphere^{d - 1}$, the geodesic is simply the great circle $d_g(\mathbf{q}_i, \mathbf{k}_j) = \arccos(\mathbf{q}_i^{\top} \mathbf{k}_j)$.
Following a similar line of reasoning, this result naturally extends to the spherical manifold.

\end{proof}
Based on the result of Proposition~\ref{prop:distance_preserving}, for architectures satisfying Assumption~\ref{ass:orthogonality} (orthogonality) in the main text, we implement the following configurations:
\begin{itemize}
    \item When $\mathbf{Q} \neq \mathbf{K}$, we set $\mathbf{W}_{\mathrm{Q}} = \mathbf{I}_d$, $\mathbf{W}_{\mathrm{K}} \in O(d)$.
    \item When $\mathbf{Q} = \mathbf{K}$, we set $\mathbf{W}_{\mathrm{Q}} = \mathbf{W}_{\mathrm{K}} = \mathbf{I}_d$.
\end{itemize}

When the numbers of attention heads $H > 1$, for each attention head, the corresponding submatrix of $\mathbf{W}_{\mathrm{Q,K}}$ is instead semi-orthogonal and no longer qualifies as an isometry, i.e., $d_g(\mathbf{q}_i, \mathbf{k}_j) \neq d_g(\mathbf{x}_i, \mathbf{x}_j)$.
See Appendix~\ref{supp_sec:spherical_fna} for a detailed discussion.

\subsection*{Convergence of Operators and PDEs}

We prove Theorem~\ref{thm:fna_inf_gen} using Lemma~\ref{lem:expansion}.
First, let us define some important constants that will appear as part of the proofs.
\begin{widetext}
\begin{equation}\label{eq:constants_1}
c_0 \coloneqq \int_{\R^d}\Phi_\alpha(\norm{\zb})\d\zb,
\qquad
\kappa_\alpha \coloneqq
\begin{cases}
C_{d,\alpha}^{-1}, & 1 \leq \alpha < 2,\\[1mm]
\pi^{d/2}/4, & \alpha=2,
\end{cases}
\qquad
\tilde c_\alpha \coloneqq \frac{\kappa_\alpha}{c_0},
\qquad c_{\alpha} \coloneqq \alpha \tilde c_\alpha,
\end{equation}
\end{widetext}
where $C_{d,\alpha}$
is the standard normalization of the fractional Laplacian in equation~\eqref{eq:frac_laplacian_euclidean}, characterized by the L\'evy symbol identity
\begin{equation}\label{eq:levy}
C_{d,\alpha}\int_{\R^d}\frac{1-\cos(\bm{\upeta}\cdot\zb)}{\Vert \zb \Vert^{d+\alpha}}\d\zb = \Vert\bm{\upeta}\Vert^\alpha,
\qquad \bm{\upeta}\in\R^d.
\end{equation}
For $1 \leq \alpha<2$ one has the closed form
$c_0 = \dfrac{2\pi^{d/2}}{\Gamma(d/2)}\cdot\dfrac{\Gamma(d)\,\Gamma(\alpha)}{\Gamma(d+\alpha)}$, while for $\alpha=2$, $c_0=\pi^{d/2}$ and hence $\tilde c_2 = 1/4$.

\begin{widetext}
\noindent\begin{minipage}{\textwidth}
\begin{lemma}[Small-$\varepsilon$ expansion of the $\Phi_\alpha$-smoothing operator]\label{lem:expansion}
Let $1 \leq \alpha \leq 2$, let $K(\zb)\coloneqq\Phi_\alpha(\Vert\zb\Vert)$, and let $f:\R^d\to\R$ satisfy
$\int(1 + \abs{\bm{\upxi}}^{\alpha})\,\abs{\hat f(\bm{\upxi})}\d\bm{\upxi} < \infty$ (e.g.\ $f \in \Sc(\R^d)$). Then, as
$h\downarrow0$,
\begin{equation}\label{eq:kernel_smoothing}
\int_{\R^d} \Phi_\alpha\!\Big(\frac{\norm{\mathbf{u} - \yb}}{h}\Big) f(\yb)\d\yb
= h^{d}\Big( c_0\, f(\mathbf{u}) - \kappa_\alpha\, h^{\alpha}\,(-\Delta)^{\alpha/2} f(\mathbf{u})
+ o(h^{\alpha}) \Big),
\end{equation}
uniformly in $\mathbf{u} \in \R^d$, where $o(\cdot)$ denotes the little-o notation.
\end{lemma}
\end{minipage}
\end{widetext} 
\smallskip

\begin{proof}

First let us define $K(\zb)=\Phi_\alpha(\Vert\zb\Vert)$ and $K_h(\zb)=h^{-d}K(\zb/h)$.
Note that the left-hand side of equation~\eqref{eq:kernel_smoothing} is $h^d\,(K_h*f)(\mathbf{u})$ and
$\widehat{K_h*f}(\bm{\upxi}) = \hat K(h\bm{\upxi})\hat f(\bm{\upxi})$. 
It therefore suffices to prove the symbol expansion
\begin{equation}\label{eq:symbol}
\hat K(\bm{\upeta}) = c_0 - \kappa_\alpha\Vert\bm{\upeta}\Vert^{\alpha} + o(\Vert\bm{\upeta}\Vert^{\alpha}),
\qquad \bm{\upeta} \to 0,
\end{equation}
together with the global bound $\abs{\hat K(\bm{\upeta}) - c_0 + \kappa_\alpha\Vert \bm{\upeta} \Vert^\alpha} \leq C\Vert \bm{\upeta} \Vert^\alpha$
for all $\bm{\upeta}$, and then use Fourier
inversion.

\textbf{Case 1} ($\alpha=2$). 
Here $K(\zb) = e^{-\Vert\zb\Vert^2}$, so
$\hat K(\bm{\upeta}) = \pi^{d/2}e^{-\Vert \bm{\upeta} \Vert^2/4} = c_0 - \frac{\pi^{d/2}}{4}\Vert \bm{\upeta} \Vert^2 + O(\Vert \bm{\upeta} \Vert^4)$,
which is \eqref{eq:symbol} with $\kappa_2 = \pi^{d/2}/4$; the global bound is immediate since
$c_0-\hat K(\bm{\upeta}) \in [0,c_0]$ and $c_0 - \hat K(\bm{\upeta}) \leq \frac{\pi^{d/2}}{4}\Vert \bm{\upeta} \Vert^2$.

\textbf{Case 2} ($1 \leq \alpha<2$). Since $K$ is radial, $\hat K$ is real and
\begin{equation}
c_0 - \hat K(\bm{\upeta}) = \int_{\R^d}\big(1-\cos(\bm{\upeta}\cdot\zb)\big)K(\zb)\d\zb .
\end{equation}
Decompose $K(\zb) = \Vert \zb \Vert^{-(d+\alpha)} + R(\zb)$ with $R(\zb) \coloneqq(1 + \Vert\zb\Vert)^{-(d + \alpha)} - \Vert \zb \Vert^{-(d + \alpha)}$. 
By the L\'evy identity in equation~\eqref{eq:levy}, the first part contributes exactly $\Vert \bm{\upeta} \Vert^{\alpha}/C_{d,\alpha} = \kappa_\alpha\Vert\bm{\upeta}\Vert^{\alpha}$. 
For the remainder, note $\abs{R(\zb)} \leq \Vert\zb\Vert^{-(d + \alpha)}$
for $\Vert\zb\Vert \leq 1$ and, by the mean value theorem,
$\abs{R(\zb)} \leq (d + \alpha)\abs{\zb}^{-(d + \alpha+1)}$ for $\Vert \zb \Vert \geq 1$. 
Using $1 - \cos t \leq \min(t^2/2,\,2)$:
\begin{widetext}
\begin{equation}
\int_{\Vert\zb\Vert \leq 1}\!\big(1-\cos(\bm{\upeta}\cdot\zb)\big)\abs{R}\d\zb
\lesssim \Vert\bm{\upeta}\Vert^2\!\!\int_{\abs{\zb} \leq 1}\!\abs{\zb}^{2-d-\alpha}\d\zb
= O(\Vert\bm{\upeta}\Vert^2),
\end{equation}
\begin{equation}
\int_{1 \leq \abs{\zb} \leq 1/\Vert\bm{\upeta}\Vert}\!\big(1-\cos(\bm{\upeta}\cdot\zb)\big)\abs{R}\d\zb
\lesssim \Vert\bm{\upeta}\Vert^2\!\int_{1}^{1/\Vert\bm{\upeta}\Vert}\! r^{-\alpha}\d r
= \begin{cases} O(\Vert\bm{\upeta}\Vert^2), & \alpha>1,\\ O\big(\Vert\bm{\upeta}\Vert^2\log\tfrac{1}{\Vert\bm{\upeta}\Vert}\big), & \alpha=1,\end{cases}
\end{equation}
and
\begin{equation}
\int_{\Vert \zb \Vert \geq 1/\Vert \bm{\upeta} \Vert}\!\abs{R}\d\zb = O(\Vert\bm{\upeta}\Vert^{\alpha+1}).
\end{equation}
\end{widetext}
All three contributions are $o(\Vert\bm{\upeta}\Vert^{\alpha})$ for $\alpha\in[1,2)$, proving \eqref{eq:symbol}.
The global bound follows from $0 \leq c_0-\hat K(\bm{\upeta}) \leq 2c_0$ together with \eqref{eq:symbol} near the origin.

By Fourier inversion,
\begin{widetext}
\begin{equation}
(K_h*f)(\mathbf{u}) - c_0 f(\mathbf{u}) + \kappa_\alpha h^{\alpha}(-\Delta)^{\alpha/2}f(\mathbf{u})
= \frac{1}{(2\pi)^d}\int e^{i\mathbf{u}\cdot\bm{\upxi}}
\Big[\hat K(h\bm{\upxi})-c_0+\kappa_\alpha\abs{h\bm{\upxi}}^{\alpha}\Big]\hat f(\bm{\upxi})\d\bm{\upxi} .
\end{equation}
\end{widetext}
The bracket is bounded by $C\abs{h\bm{\upxi}}^{\alpha}$ uniformly and is $o(\abs{h\bm{\upxi}}^{\alpha})$ pointwise as $h\to0$; since $\int\abs{\bm{\upxi}}^{\alpha}\abs{\hat f}\d\bm{\upxi} < \infty$, dominated convergence yields that the right-hand side is $o(h^{\alpha})$, uniformly in $\mathbf{u}$ (the bound does not depend on $\mathbf{u}$).

\end{proof}

\begin{proof}[Proof of Theorem~\ref{thm:fna_inf_gen}]

For the Gaussian case, there have previously been methods that show this proof through the heat equation \cite{Coifman2006,Evans2022}.
We use an analytical method that proves the result for $\manifold = \R^d$.
We obtain the expression for its Taylor expansion of FNA map w.r.t $\varepsilon$. 
\begin{widetext}
\begin{equation}
\begin{aligned}
& T_{\rho,\varepsilon}^{\mathrm{FNA}}[f](\mathbf{x})
\coloneqq \frac{\displaystyle\int_{\R^d} \Phi_{\alpha}\rbrac{ \Vert \mathbf{x} - \mathbf{y} \Vert/\varepsilon^{1/\alpha}} f\left(\mathbf{y}\right) \rho(\mathbf{y}) \d \mathbf{y}}{\displaystyle\int_{\R^d} \Phi_{\alpha}\rbrac{ \Vert \mathbf{x} - \mathbf{x}' \Vert/\varepsilon^{1/\alpha}} \rho(\mathbf{x}') \d \mathbf{x}'} \\
&= \frac{c_0(f \rho)\left( \mathbf{x} \right) - \varepsilon \kappa_{\alpha} (-\Delta)^{\alpha/2}(f \rho) \left( \mathbf{x} \right) + o(\varepsilon)}{c_0\rho(\mathbf{x}) - \varepsilon \kappa_{\alpha} (-\Delta)^{\alpha/2} \rho(\mathbf{x}) + o(\varepsilon)} \\
\stackrel{(\star)}{=} & \sbrac{ c_0 (f \rho)\left( \mathbf{x} \right) - \varepsilon \kappa_{\alpha} (-\Delta)^{\alpha/2}(f \rho) \left( \mathbf{x} \right) + o(\varepsilon) } \frac{1}{c_0 \rho(\mathbf{x})} \left(1 + \varepsilon \tilde{c}_{\alpha} \frac{(-\Delta)^{\alpha/2} \rho(\mathbf{x})}{\rho(\mathbf{x})} + o(\varepsilon)\right) \\
= & f(\mathbf{x}) - \frac{\varepsilon \tilde{c}_{\alpha} (-\Delta)^{\alpha/2}(f\rho)(\mathbf{x}) }{\rho(\mathbf{x})} + \varepsilon \tilde{c}_{\alpha} f(\mathbf{x}) \frac{(-\Delta)^{\alpha/2} \rho(\mathbf{x})}{\rho(\mathbf{x})} + o(\varepsilon) \\
= & f(\mathbf{x}) - \varepsilon \frac{\tilde{c}_{\alpha}}{\rho(\mathbf{x})} \sbrac{ (-\Delta)^{\alpha/2}(f \rho)(\mathbf{x}) - f(\mathbf{x}) (-\Delta)^{\alpha/2}(\rho)(\mathbf{x}) } + o(\varepsilon).
\end{aligned}
\end{equation}
By substituting $h = \varepsilon^{1/\alpha}$ into equation~\eqref{eq:kernel_smoothing} , the second equality follows from an application of Lemma~\ref{lem:expansion} to the numerator with the test function $f \rho$ and to the denominator with $\rho$.
The third equality $(\star)$ is obtained through the expansion of the function $h(z) = 1/z$ applied to:
\begin{equation}
\begin{aligned}
\frac{1}{c_0\rho - \varepsilon \kappa_{\alpha} (-\Delta)^{\alpha/2} \rho + o(\varepsilon)}
= \frac{1}{c_0\rho}\left(1 + \varepsilon \tilde{c}_{\alpha} \frac{(-\Delta)^{\alpha/2} \rho}{\rho} + o(\varepsilon)\right).
\end{aligned}
\end{equation}
\end{widetext}
Dividing by $\varepsilon$ and passing to the limit yields
\begin{equation}
\mathcal{H}^{\mathrm{FNA}}[f] 
\coloneqq \lim_{\varepsilon \downarrow 0} \frac{T_{\rho,\varepsilon}^{\mathrm{FNA}}[f] - f}{\varepsilon}.    
\end{equation}
    
\end{proof}

\subsection*{Scaling laws for FNA}

To characterize the dependence of FNA on the head dimension, we consider
the ansatz $\kappa(d_h)=d_h^\gamma$ in
equation~\eqref{eq:frac_attn_score}, following previous work \cite{Giuseppe2026}.
For fixed $\alpha\in[1,2)$, let
\begin{equation}
\begin{aligned}
R_{ij} &\coloneqq \lVert \mathbf{q}_i-\mathbf{k}_j\rVert,
\\
C_{ij} &= \left(1+\frac{R_{ij}}{d_h^\gamma}\right)^{-(d_h + \alpha)},
\end{aligned}
\end{equation}
and denote the row-normalized attention weights by
$A_{ij} = C_{ij}/\sum_{\ell=1}^n C_{i\ell}$.

We consider the limit $d_h \to \infty$ with a fixed number $n$ of
eligible keys. For each query, we assume that the distances concentrate
around a common radius,
\begin{equation}
\begin{aligned}
R_{ij} &= r_i(d_h)+\delta_{ij},
\\
r_i(d_h) &\sim c_i\sqrt{d_h},
\\
\delta_{ij} &= O_{\mathbb{P}}(1),
\end{aligned}
\end{equation}
where $c_i > 0$.
To distinguish the limiting attention regimes, we further assume that
the centered distance vector
$(\delta_{i1},\ldots,\delta_{in})$
converges in distribution to a finite random vector
$(Z_{i1},\ldots,Z_{in})$ with an almost surely unique minimum.

Since row-normalization leaves weight ratios unchanged, the mean value
theorem gives
\begin{equation}
\begin{aligned}
\label{eq:multiscale_critical_distance_scaling}
\log\frac{A_{ij}}{A_{ik}} 
&= \log\frac{C_{ij}}{C_{ik}}\\
&= -\frac{d_h + \alpha}{d_h^\gamma + \xi_{ijk}}
(R_{ij} - R_{ik}),
\end{aligned}
\end{equation}
where $\xi_{ijk} \in (R_{ij}, R_{ik})$.
The concentration assumption implies that the sensitivity to
within-row distance differences is asymptotically determined by
\begin{equation}
\begin{aligned}
\beta_i(d_h) &\coloneqq \frac{d_h+\alpha}{d_h^\gamma+r_i(d_h)}\\
&\sim
\begin{cases}
d_h^{1-\gamma},
& \gamma>\tfrac12,\\[3pt]
\dfrac{\sqrt{d_h}}{1+c_i},
& \gamma=\tfrac12,\\[7pt]
\dfrac{\sqrt{d_h}}{c_i},
& \gamma<\tfrac12.
\end{cases}
\end{aligned}
\end{equation}
Thus, the limit $\lim_{d_h \rightarrow \infty}\beta_i(d_h)$ diverges for $\gamma<1$, approaches unity
for $\gamma=1$, and vanishes for $\gamma>1$.
These limits yield three asymptotic regimes:
\begin{equation}
\begin{aligned}
\gamma < 1
&\quad\Longrightarrow\quad
\text{sharp attention},\\
\gamma = 1
&\quad\Longrightarrow\quad
\text{non-degenerate (critical) attention},\\
\gamma > 1
&\quad\Longrightarrow\quad
\text{uniform attention}.
\end{aligned}
\end{equation}
More precisely, $\max_j A_{ij}\to1$ in probability for $\gamma<1$,
whereas $A_{ij} \to 1/n$ in probability for $\gamma > 1$.
At $\gamma=1$, the attention row converges in distribution to
\begin{equation}
(A_{ij})_{j=1}^{n} \xrightarrow{\mathrm{d}}
\left( \frac{\exp(-Z_{ij})}{\sum_{\ell=1}^{n}\exp(-Z_{i\ell})} \right)_{j=1}^{n}.
\end{equation}
Here, critical refers to the preservation of finite,
non-vanishing within-row log-score contrasts.
Under the stated concentration and fixed-$n$ assumptions, this
identifies $\kappa \propto d_h$ as a natural dimension-dependent
scaling, rather than a task-independent optimum.

\subsection*{Experiments}

\textbf{Benchmark datasets and potential biases.}
We evaluate FNA and other related models on IMDb and AG News text classification, CIFAR-10 image classification, and Multi30K English-to-German translation. 
These benchmarks cover selected task domains, languages and label spaces and reflect the collection and annotation procedures of their original sources. Potential sampling and annotation biases, together with limitations in domain coverage, may influence the patterns learned by the models and their measured performance. 
Evaluation across several tasks provides complementary tests of the proposed attention mechanism, but does not, by itself, demonstrate mitigation of dataset-specific biases. 
Accordingly, the empirical findings are interpreted within the evaluated datasets, model configurations and training protocols; generalization to other populations, domains or deployment conditions requires further evaluation.

\smallskip
\noindent
\textbf{Hardware.}
All model training and pretrained analysis were realized on Tesla V100-SXM2-32GB GPUs. 

\smallskip
\noindent
\textbf{Text classification.}
We conducted two sets of model training experiments for text classification as described in the main text: shallow models of depth 1 (Fig.~\ref{fig:imdb_ablation}), and deeper 6-layer models (Fig.~\ref{fig:train_imdb}).

In Fig.~\ref{fig:imdb_ablation}, we focus on single-head attention and single-layer models ($H = 1$, $L = 1$) that satisfy Assumption~\ref{ass:orthogonality}, and further examine both cases where Assumption~\ref{ass:symmetry} holds in Fig.~\ref{fig:imdb_ablation}\hyperref[fig:imdb_ablation]{b} and where it does not in Fig.~\ref{fig:imdb_ablation}\hyperref[fig:imdb_ablation]{a}.
We perform a parameter sweep over embedding dimensions $d \in \{8,16,32,64, 128, 256\}$, with fully connected hidden layers fixed at 256.
Each configuration is trained for 20 epochs across 5 independent runs using Adam \cite{Kingma2015} with batch size 16.
For $d \in \{8, 16, 32, 64\}$, the initial learning rate of $1\times10^{-3}$ decayed by a factor of 5 at epoch 15; similarly for $d \in \{128, 256\}$, the initial learning rate of $5\times10^{-4}$ decayed by a factor of 5 at epoch 15.
The FNA models are trained for fractional orders $\alpha \in \{1, 1.2, 1.4, 1.6, 1.8, 2\}$.

In Fig.~\ref{fig:train_imdb}, all models have 6 layers, 8 attention heads, an embedding dimension of 256 and are trained for 20 epochs. We use an initial learning rate of $10^{-4}$ and decay it by a factor of 10 at epoch 15. The batch size is set to 32.

For both sets of text classification experiments, we set the maximal sequence length to 512 and the momentum terms in Adam to $(\beta_1, \beta_2) = (0.9, 0.999)$. No weight decay is deployed.
For all text classification tasks, we use the distance scaling $\kappa = d_h^{3/2}$ for multiscale FNA ($\alpha < 2$) where $d_h = d / H$ is the head dimension. 
We also set $\dman = d_h$ in $\Phi_{\alpha}$ (equation~\eqref{eq:hk_dichotomy_2}), consistent with the canonical choice in fractional diffusion maps \cite{Antil2021} when $H = 1$.
For $\alpha = 2$, we set $\kappa = \sqrt{d_h}$ for this and all other task types.
This choice aligns with prior practices using the dot product \cite{Vaswani2017} or (squared) Euclidean distance \cite{Kim2021}, ensuring that the resulting values do not become excessively large and push the softmax function into regions with vanishing gradients.

\smallskip
\noindent
\textbf{Attention graph node ablation.}
In Fig.~\ref{fig:dynamic_inference_all}, we consider the Bernoulli probability $p = 0, 0.1, 0.2,  \cdots, 1$, of random node removal by masking its corresponding index location in the attention score.
This experiment is carried out on all pretrained single-layer Transformers, Sinkformers and FNA models in Fig.~\ref{fig:imdb_ablation}.
During inference, the batch size is set to 64.


\smallskip
\noindent
\textbf{Markov chain diagonalization.}
This process is only carried out for the ideal case where there is one attention head and $\mathbf{W}_{\mathrm{Q}} = \mathbf{W}_{\mathrm{K}} = \mathbf{I}_d$ to qualify for the implementation of diffusion maps, which will be elaborated on in the next two subsections.
Let us define the diagonal matrix $D_{ii} = \sum_{j} C_{ij}$, then the attention weight matrix $\mathbf{A}$ in equation~\eqref{eq:frac_attn_weights} is similar to the symmetrical matrix \cite{Nadler2008}
\begin{equation}
\tilde{\mathbf{A}} \coloneqq \mathbf{D}^{-1/2} \mathbf{C} \mathbf{D}^{-1/2}.
\end{equation}
Thus, $\mathbf{A}$ shares the same set of eigenvalues as $\tilde{\mathbf{A}}$, all of which are real.
For a connected graph, the shared eigenvalues are ordered as
\begin{equation}
1 = \eta_0 \geq \eta_1 \geq \cdots \geq \eta_{n-1}.
\end{equation}
If floating point errors cause $\tilde{\mathbf{A}}$ to be asymmetrical, the matrix $(\tilde{\mathbf{A}} + \tilde{\mathbf{A}}^{\top})/2$ can be used for diagonalization in its place to guarantee numerical stability.
Additionally, the eigenvectors of $\tilde{\mathbf{A}}$ form an orthonormal basis $\{ \mathbf{v}_j \}_{j=0}^{n-1}$.
The left and right eigenvectors of $\mathbf{A}$, denoted $\bm{\upphi}_j$ and $\bm{\uppsi}_j$ are related to $\mathbf{v}_j$ through
\begin{equation}
\begin{aligned}
\bm{\upphi}_j &= \mathbf{D}^{1/2} \mathbf{v}_j / \sqrt{S}, \\
\bm{\uppsi}_j &= \sqrt{S} \mathbf{D}^{-1/2} \mathbf{v}_j,
\end{aligned}
\end{equation}
where $S$ is defined in equation~\eqref{eq:att_score_sum}.
Note they are also bi-orthonormal, i.e., $\left\langle \bm{\upphi}_i, \bm{\uppsi}_j\right\rangle = \delta_{i, j}$.

\smallskip
\noindent
\textbf{FNA operator eigenspectrum and graph connectivity.}
For the eigenspectrum analysis in
Fig.~\ref{fig:fna_token_and_spectrum}\hyperref[fig:fna_token_and_spectrum]{a}, we use 500 sample indices with angles $\theta_i=2\pi i/499$ for $i=0,\ldots,499$. The calculation retains both endpoints, which represent the same point on $\mathbb{S}^1$. 
We use the geodesic distance
\begin{equation}
d_g(\mathbf{x}_i,\mathbf{x}_j)
=
\arccos(\mathbf{x}_i^\top\mathbf{x}_j),
\end{equation}
and correspondingly compute
\begin{equation}\label{eq:spherical_spectral_weights}
C_{ij}^{\varepsilon}
=
\Phi_{\alpha}\left(
\frac{d_g(\mathbf{x}_i,\mathbf{x}_j)}
{\varepsilon^{1/\alpha}}
\right),
\qquad
\mathbf{A}^{\varepsilon}
=
N_R(\mathbf{C}^{\varepsilon}),
\end{equation}
using intrinsic dimension $d_{\mathcal M}=1$ in the
power-law profile.
The spatial bandwidth is fixed at
$\varepsilon^{1/\alpha}=10^{-2}$.

For each positive eigenvalue $\eta_i$ of
$\mathbf{A}^{\varepsilon}$, we define
\begin{equation}\label{eq:spherical_decay_rates}
\widehat{\lambda}_i^\varepsilon
\coloneqq
-\varepsilon^{-1}\log\eta_i.
\end{equation}

By construction,
\begin{equation}
\mathbf{A}^{\varepsilon}\bm{\uppsi}_i
=
\eta_i\bm{\uppsi}_i
=
e^{-\varepsilon\widehat{\lambda}_i^{\varepsilon}}
\bm{\uppsi}_i.
\end{equation}

These quantities estimate the decay rates of the limiting
operator when the sampling and bandwidth limits are sufficiently
resolved.
If $-\mathcal{H}\varphi_i=\lambda_i\varphi_i$, the associated limiting
semigroup satisfies
\begin{equation}
e^{t\mathcal{H}}\varphi_i
=
e^{-t\lambda_i}\varphi_i.
\end{equation}
These logarithmic spectral estimates are plotted in
Fig.~\ref{fig:fna_token_and_spectrum}\hyperref[fig:fna_token_and_spectrum]{a} under a uniform grid on $\mathbb{S}^1$.
For $\alpha=2$, the limiting operator is
$\tfrac14(-\Delta_{\mathbb{S}^1})$.
For $\alpha<2$, the geodesic power-law kernel yields
a nonlocal integral operator whose eigenvalues have
leading growth proportional to $j^\alpha$.
We compare these spectral estimates with the reference
scalings $j^2$ and $j^\alpha$, respectively \cite{Antil2021}.

The connectivity examples in
Fig.~\ref{fig:fna_token_and_spectrum}\hyperref[fig:fna_token_and_spectrum]{b}--%
\hyperref[fig:fna_token_and_spectrum]{c}
are separate experiments in $\mathbb{R}^2$ and use
Euclidean distance.
For clarity, only attention weights exceeding
$3\times10^{-4}$ are displayed.

\smallskip
\noindent
\textbf{Diffusion maps for FNA.}
Diffusion map analyses are performed under Assumptions~\ref{ass:orthogonality} and \ref{ass:symmetry} stated in the main text.
Let us define the context window size $n_{\max}$.
For sequences with the maximum length $n = n_{\max}$, the entire sequence is retained.
For shorter sequences ($n < n_{\max}$), as padding masks are applied to both the final $(n_{\max} - n)$ queries and keys, we truncate all padding indices and operate on the resulting $n \times n$ fractional attention score matrix $\mathbf{C}$ (equation~\eqref{eq:frac_attn_score}) in order to guarantee the symmetry condition $\mathbf{Q} = \mathbf{K}$.

The diffusion map embeddings in Fig.~\ref{fig:fna_dm}\hyperref[fig:fna_dm]{c} are obtained through the diffusion maps algorithm with the top two components in equation~\eqref{eq:diffusion_map} for the words ``much'' and ``danger''.
Note that the attention scores basically maps $\mathbf{x}_i \mapsto \tilde{\mathbf{x}}_i \coloneqq \mathbf{x}_i / \kappa$.
Therefore we directly perform the diffusion maps analysis on the sequential data points $\{ \tilde{\mathbf{x}}_i \}_{i=1}^n$.

The diagonalization method for the fractional self-attention weights (equation~\eqref{eq:frac_attn_weights}) is detailed above. 
This allows us to establish the eigen-decomposition (equation~\eqref{eq:attn_weight_decomposition}) where $\bm{\upphi}_i$ and $\bm{\uppsi}_i$ are the left and right eigenvectors of $\mathbf{A}$; $\eta_i$ are the corresponding eigenvalues.
The diffusion distance between tokens $i$ and $j$, we have \cite{Coifman2006,Nadler2008}
\begin{align}\label{eq:diffusion_dist}
\begin{aligned}
D_\tau^2(i,j)
&=\sum_{k=1}^n \frac{[p(\tau,k\mid i)-p(\tau,k\mid j)]^2}{\pi_k} \\
&=\sum_{\ell=1}^{n-1}\eta_\ell^{2\tau}
  [\uppsi_\ell(i)-\uppsi_\ell(j)]^2.
\end{aligned}
\end{align}
Here $\pi_k$ is the stationary probability of token $k$. The second line follows from equation~\eqref{eq:attn_weight_decomposition} and orthonormality of the right eigenvectors under the stationary measure. The stationary term cancels because $\uppsi_0$ is constant. Denote the map in equation~\eqref{eq:diffusion_map} with $m$ coordinates by $\bm{\Psi}^{(m)}_{\tau,\varepsilon}$. Using all $n-1$ nonstationary modes gives
\begin{equation}
\begin{aligned}\label{eq:diffusion_dist_approx}
&D_\tau^2(i,j)\\
&=
\left\Vert \bm{\Psi}^{(n-1)}_{\tau,\varepsilon}(i)-
\bm{\Psi}^{(n-1)}_{\tau,\varepsilon}(j)\right\Vert^2.
\end{aligned}
\end{equation}
The distance computed from the first two coordinates retains two nonnegative terms of the sum and therefore gives a lower bound on $D_\tau(i,j)$.

\smallskip
\noindent
\textbf{Data-informed parameter selection.}
We utilize the log-log criterion based on the original work of diffusion maps for Gaussian kernels \cite{Coifman2006,Singer2009}. 
Define the following summation
\begin{equation} \label{eq:att_score_sum}
S(\varepsilon) = \sum_{i,j=1}^n \Phi_{\alpha}\left( \frac{\Vert \tilde{\xb}_i - \tilde{\xb}_j \Vert}{\varepsilon^{1/\alpha}} \right).
\end{equation}
As $\varepsilon \downarrow 0$, only the diagonal terms survive and thus $S \rightarrow n$. As $\varepsilon \uparrow \infty$, all terms approach 1 so $S \rightarrow n^2$.
In a plot of $\log S(\varepsilon)$ against $\log \varepsilon$, one would expect a curve which has flat ends and a linear regime in between and $\frac{\d \log S}{\d \log \varepsilon}$ is
at its maximum.
We base our selection of $\varepsilon$ on the above in Fig.~\ref{fig:fna_dm}\hyperref[fig:fna_dm]{a--b} and set the diffusion time $\tau = 15$ in Fig.~\ref{fig:fna_dm}\hyperref[fig:fna_dm]{c}.

\smallskip
\noindent
\textbf{Visual processing.}
In Fig.~\ref{fig:train_cifar10}, the models with 4 layers and 6 attention heads are trained for 110 epochs on the CIFAR10 dataset where the patch size is set to 4.
For unconstrained query and key projection matrices $\mathbf{W}_{\mathrm{Q,K}} \notin O(d)$, we use Adam with an initial learning rate of $5 \times 10^{-4}$ and decay it by a factor of 10 at epoch 100. 
For $\mathbf{W}_{\mathrm{Q,K}} \in O(d)$, we use Adam with an initial learning rate of $6 \times 10^{-4}$ and decay it by a factor of 5 at epoch 100. 
Batch size is set to 64 for both settings.
The momentum terms are set as $(\beta_1, \beta_2) = (0.9, 0.95)$. 
No weight decay is deployed.
In equation~\eqref{eq:frac_attn_score}, we set $\kappa=d_h$ for multiscale FNA ($\alpha<2$) and $\kappa=\sqrt{d_h}$ for Gaussian FNA ($\alpha=2$).

\smallskip
\noindent
\textbf{Translation.}
Each model contains 4 encoding and 4 decoding layers with 8 attention heads.
We set $\kappa = d_h^{1/2}$ for both multiscale and local attention in equation~\eqref{eq:frac_attn_score} where $d = 64$.
All models are trained for 40 epochs with Adam using a cosine learning rate scheduler, including 1,000 warm-up steps at the start of training.

\section*{Data Availability}

All referenced datasets are publicly available: IMDB \cite{Maas2011} (\url{https://huggingface.co/datasets/stanfordnlp/imdb}), AG News \cite{Zhang2015} {\url{https://huggingface.co/datasets/sh0416/ag_news}}, CIFAR-10 \cite{cifar10} (\url{https://www.cs.toronto.edu/~kriz/cifar.html}) and Multi30k \cite{Elliott2016} (\url{https://huggingface.co/datasets/bentrevett/multi30k}). 
Source data are available with this paper.

\section*{Code Availability}
The code for simulations and analyses is available without restrictions on GitHub (\url{https://github.com/CKQu1/fractional-attn}) and Zenodo \cite{Qu_fna_code}.


\section*{Acknowledgements}
We thank Guozhang Chen, Yang Qi, and Asem Wardak for valuable discussions that contributed to this work.
We also acknowledge the National Computational Infrastructure (NCI) for providing the computational resources used to train the models.

\section*{Funding Statement}
This work was supported by the Australian Research Council under grant DP160104368 (P.G.).

\section*{Author Contributions}
C.K.Q., A.L. and P.G. designed the study, performed the research, and wrote the paper.

\section*{Competing Interests}
The authors declare no competing interests.



\setcounter{figure}{0}
\renewcommand{\thefigure}{A\arabic{figure}}

\appendix
\section{Multi-head FNA dynamics}  \label{supp_sec:multihead_fna}

We first extend the ResNet equation (equation~(1) in the main text) to multi-head FNA:
\begin{widetext}
\begin{equation} \label{eq:mh_sa_resnet}
\mathbf{x}_i^{l+1} = \mathbf{x}_i^{l} + \sum_{h=1}^H \beta_h \mathbf{W}_{\mathrm{O}}^{(h)} \sum_{j=1}^n \frac{ \Phi_{\alpha}(\Vert \mathbf{W}_{\mathrm{Q}}^{(h)} \mathbf{x}_i - \mathbf{W}_{\mathrm{K}}^{(h)} \mathbf{x}_j \Vert / \varepsilon^{1/\alpha} ) }{\sum_{k=1}^n  \Phi_{\alpha}(\Vert \mathbf{W}_{\mathrm{Q}}^{(h)} \mathbf{x}_i - \mathbf{W}_{\mathrm{K}}^{(h)} \mathbf{x}_k \Vert / \varepsilon^{1/\alpha})} \mathbf{W}_{\mathrm{V}}^{(h)} \mathbf{x}_j^{l},
\end{equation} 
\end{widetext}
where $\mathbf{W}_{\mathrm{Q,K}}^{(h)} \in \R^{d_h \times d}$, $\mathbf{W}_{\mathrm{V}}^{(h)} \in \R^{d_v \times d}$ and $\mathbf{W}_{\mathrm{O}}^{(h)} \in \R^{d \times d_v}$ for each attention head $h = 1, \cdots, H$.
The constant $\beta_h > 0$ denotes the weight of attention head $h$ which could be absorbed into $\mathbf{W}_{\mathrm{O}}$, but it is kept for generality. 
Similar to equation~(2) in the main text, the underlying ODE can be described for multi-head FNA.
Under the assumptions below, we derive a formal limit of the multi-head FNA continuity equation and identify a case in which it reduces to a fractional diffusion equation.

Theorem~2 in the main text identifies the infinitesimal generator of single-head FNA. 
We extend the analysis to multiple heads under analogous assumptions on the projection matrices $\mathbf{W}_{\mathrm{Q,K}}^{(h)}$ and derive the corresponding PDE using the linear value map $f(\mathbf{x}) \coloneqq \mathbf{W}_{\mathrm{V}}\mathbf{x}$ as the test function for each head.
Following the theoretical framework for fractional neural attention (FNA) in the main text, we define the FNA kernel for each attention head $h = 1,\cdots, H$,
\begin{widetext}
\begin{equation} \label{eq:multihead_fna_kernel}
k_{\varepsilon}^{(h)}(\mathbf{x}, \mathbf{y}) \coloneqq \frac{\Phi_{\alpha}(\Vert \mathbf{W}_{\mathrm{Q}}^{(h)} \mathbf{x} - \mathbf{W}_{\mathrm{K}}^{(h)} \mathbf{y} \Vert / \varepsilon^{1/\alpha})}{\displaystyle\int_{\R^d} \Phi_{\alpha}(\Vert \mathbf{W}_{\mathrm{Q}}^{(h)} \mathbf{x} - \mathbf{W}_{\mathrm{K}}^{(h)} \mathbf{x}' \Vert / \varepsilon^{1/\alpha})  \rho(\mathbf{x}') \d \mathbf{x}'},  
\end{equation}
with $\Phi_{\alpha}$ given by equation~(10) for $\alpha=2$ and equation~(11) for $1\leq\alpha<2$ in the main text. 
Using the aggregate matrix $\mathbf{U}_h\coloneqq\WO\WV$, we define the multi-head attention map and its associated vector fields:
\begin{flalign}
&\text{(attention map)}
& T_{\rho,\varepsilon}(\mathbf{x})
&\coloneqq
\sum_{h=1}^H \beta_h
\int_{\R^d} k_{\varepsilon}^{(h)}(\mathbf{x},\mathbf{y})
\mathbf{U}_h \mathbf{y}\rho(\mathbf{y})\,\d\mathbf{y}, && \label{eq:attention_map}
\\
&\text{(vector field)}
& \overline{T}_{\rho,\varepsilon}(\mathbf{x})
&\coloneqq
\frac{1}{\varepsilon}
\sum_{h=1}^H \beta_h
\int_{\R^d} k_{\varepsilon}^{(h)}(\mathbf{x},\mathbf{y})
\mathbf{U}_h\mathbf{y}\rho(\mathbf{y})\,\d\mathbf{y}, && \label{eq:vector_field}
\\
&\text{(recentered field)}
& \widetilde{T}_{\rho,\varepsilon}(\mathbf{x})
&\coloneqq
\overline{T}_{\rho,\varepsilon}(\mathbf{x})
+\frac{1}{\varepsilon}
\sum_{h=1}^H \beta_h
\bigl(\WQ\bigr)^{\!\top}\WQ\,\xb . && \label{eq:recentered_field}
\end{flalign}
\end{widetext}
As $\varepsilon \downarrow 0$, $T_{\rho,\varepsilon}(\mathbf{x}) = -\sum_h \beta_h (\mathbf{W}_{\mathrm{Q}}^{(h)})^{\top}\mathbf{W}_{\mathrm{Q}}^{(h)}\mathbf{x} + O(\varepsilon)$, so the vector field $\overline{T}_{\rho,\varepsilon} = T_{\rho,\varepsilon}/\varepsilon$ contains a term of order $1/\varepsilon$. 
The recentering term in equation~\eqref{eq:recentered_field} cancels this leading term, and the $o(\varepsilon)$ remainder vanishes after division by $\varepsilon$. The resulting limit $\widetilde{T}_{\rho,0} \coloneqq \lim_{\varepsilon\downarrow 0}\widetilde{T}_{\rho,\varepsilon}$ appears in the following theorem.

Generally, both the full query and key projections are linear maps from $\R^d$ to $\R^d$. In a conventional multi-head implementation, each head uses a map from $\R^d$ to $\R^{d_h}$, where $H$ divides $d$ and $d_h\coloneqq d/H$.
The PyTorch function \texttt{torch.\allowbreak nn.\allowbreak utils.\allowbreak parametrizations.\allowbreak orthogonal} can enforce orthogonality of $\mathbf{W}_{\mathrm{Q}}$ throughout training. 
The query projection for head $i$ is $\mathbf{W}_{\mathrm{Q}}^{(i)}\mathbf{x} \in \R^{d_h}$, where $\mathbf{W}_{\mathrm{Q}}^{(i)}$ is the submatrix $[\mathbf{W}_{\mathrm{Q}}]_{(i-1)d_h +1: i d_h, :}$ for $i = 1,\cdots,H$.
When $\mathbf{W}_{\mathrm{Q}}$ is orthogonal, each submatrix $\mathbf{W}_{\mathrm{Q}}^{(i)}$ has orthonormal rows. If $d_h<d$, the corresponding per-head projection can change Euclidean distances on $\R^d$.
If we set $d_h=d$, each $\mathbf{W}_{\mathrm{Q,K}}^{(h)}$ is square, so we can additionally require invertibility.
In this case, we can obtain Theorem~\ref{thm:fna_pde} without requiring orthogonality of the projection matrices. 

\begin{assumption}\label{ass:main}
For each head $h=1,\dots,H$:
\begin{enumerate}
    \item[(i)] $d_h=d_v=d$ and $\WQ,\WK\in\R^{d\times d}$ are invertible;
\item[(ii)] $\Ub_h \coloneqq \WO\WV = -\big(\WQ\big)^{\!\top}\WK$;
\item[(iii)] $\rho$ is a strictly positive Schwartz function on $\R^d$.
\end{enumerate}
Under Assumption~\ref{ass:main}(\hyperref[ass:main]{i}) we write $\tilde\rho_h \coloneqq \rho\circ\big(\WK\big)^{-1}$, which is again a strictly positive Schwartz function.
In particular, $\tilde\rho_h\big(\WQ\xb\big) = \rho\big((\WK)^{-1} \WQ \xb\big)$.
\end{assumption}
\smallskip

\begin{theorem}[PDE of multi-head FNA] \label{thm:fna_pde}

Let $\manifold=\R^d$ and suppose Assumption~\ref{ass:main} holds. In the formal limit $\varepsilon\to0$, the recentered multi-head FNA continuity equation becomes
\begin{equation} \label{eq:limit_pde}
\partial_t \rho + \nabla \cdot (\rho \widetilde{T}_{\rho,0}) = 0.
\end{equation}
where the limiting multi-head vector field is
\begin{widetext}
\begin{equation} \label{eq:fna_velocity_field}
\widetilde{T}_{\rho,0}(\xb)
= \tilde c_\alpha \sum_{h=1}^H \beta_h\,\big(\WQ\big)^{\!\top}\,
\frac{\big[(-\Delta)^{\alpha/2}(\yb\,\tilde\rho_h) - \yb\,(-\Delta)^{\alpha/2}\tilde\rho_h\big]\big(\WQ\xb\big)}{\tilde\rho_h\big(\WQ\xb\big)},
\end{equation}
where the constant $\tilde{c}_{\alpha}$ depends on $\alpha$.
Equivalently, 
\begin{equation}\label{eq:fna_field_limit_comm}
\widetilde{T}_{\rho, \varepsilon}(\mathbf{x}) \;\xrightarrow[\;\varepsilon\to0\;]{}\; \widetilde{T}_{\rho, 0}(\mathbf{x})
= -\,\alpha\,\tilde c_\alpha \sum_{h=1}^H \beta_h\,\big(\WQ\big)^{\!\top}\,
\frac{\big[\nabla(-\Delta)^{\alpha/2-1}\tilde\rho_h\big]\big(\WQ\xb\big)}{\tilde\rho_h\big(\WQ\xb\big)}.
\end{equation}
\end{widetext}
    
\end{theorem}

Equation~\eqref{eq:fna_velocity_field} uses invertible query and key projections and the condition on $\mathbf{U}_h$ in Assumption~\ref{ass:main}(\hyperref[ass:main]{ii}) in place of orthogonality and symmetry (Assumptions~A and~B in the main text). The following lemma converts the commutator in equation~\eqref{eq:fna_velocity_field} into the gradient form in equation~\eqref{eq:fna_field_limit_comm}. 

\begin{lemma}
\label{lemma:comm_relation}

For the fractional Laplacian on $\manifold=\R^d$, with $\alpha\in[1,2]$ and $\rho$ as in Assumption~\ref{ass:main}(\hyperref[ass:main]{iii}),
\begin{equation}
\begin{aligned} \label{eq:frac_lapl_identity}
&(-\Delta)^{\alpha / 2}(\mathbf{x} \rho(\mathbf{x})) - \mathbf{x}(-\Delta)^{\alpha / 2} \rho(\mathbf{x})\\
&= -\alpha \nabla (-\Delta)^{\alpha / 2 - 1} \rho(\mathbf{x}).
\end{aligned}
\end{equation}

\end{lemma}

\begin{proof}

First, given $\mathbf{x} \in \R^d$, the fractional Laplacian $(-\Delta)^{\alpha / 2}$ acts componentwise on the vector-valued function $\mathbf{x}\rho(\mathbf{x})$:
\begin{equation}
(-\Delta)^{\alpha / 2}(\mathbf{x} \rho(\mathbf{x})) 
=\left(\begin{array}{c}
(-\Delta)^{\alpha / 2}\left(x_1 \rho\right)(\mathbf{x}) \\
\vdots \\
(-\Delta)^{\alpha / 2}\left(x_d \rho\right)(\mathbf{x})
\end{array}\right).
\end{equation}
The $j$th component of the commutator is
\begin{equation}
\begin{aligned}
&\left[(-\Delta)^{\alpha / 2}(\mathbf{x} \rho(\mathbf{x})) - \mathbf{x}(-\Delta)^{\alpha / 2} \rho(\mathbf{x})\right]_j\\
&= (-\Delta)^{\alpha / 2}\left(x_j \rho\right)(\mathbf{x})-x_j(-\Delta)^{\alpha / 2} \rho(\mathbf{x}) \\
\end{aligned}
\end{equation}
Now set $s = \alpha/2$ and use the Fourier convention
\begin{equation}
\hat{f}(\xib) = \int_{\R^d} e^{-i \mathbf{x} \cdot \xib} f(\mathbf{x}) \d \mathbf{x}.
\end{equation}
Then
\begin{equation}
\begin{aligned}
&\widehat{(-\Delta)^s f}(\xib)=|\xib|^{2s}\widehat f(\xib)=|\xib|^\alpha\widehat f(\xib),
\\
&
\widehat{x_j f}(\xib)=i\partial_{\xi_j}\widehat f(\xib).
\end{aligned}
\end{equation}
Therefore 
\begin{equation}
\begin{aligned}
\mathcal F\!\big[(-\Delta)^s(x_j\rho)\big](\xib)
&=
|\xib|^\alpha i\partial_{\xi_j}\widehat\rho(\xib),\\
\mathcal F\!\big[x_j(-\Delta)^s\rho\big](\xib)
&=
i\partial_{\xi_j}\left(|\xib|^\alpha\widehat\rho(\xib)\right)\\
&=
i\alpha |\xib|^{\alpha-2}\xi_j\widehat\rho(\xib)
+
i|\xib|^\alpha\partial_{\xi_j}\widehat\rho(\xib).
\end{aligned}
\end{equation}
Subtracting gives
\begin{equation}
\begin{aligned}
&\mathcal F\!\left[
(-\Delta)^s(x_j\rho)-x_j(-\Delta)^s\rho
\right](\xib)\\
&= -i\alpha \xi_j|\xib|^{\alpha-2}\widehat\rho(\xib).
\end{aligned}
\end{equation}
Since
\begin{equation}
\mathcal F\!\left[\partial_j(-\Delta)^{\alpha/2-1}\rho\right](\xib)
=
i\xi_j|\xib|^{\alpha-2}\widehat\rho(\xib),
\end{equation}
we obtain the component-wise formula
\begin{equation}
(-\Delta)^{\alpha/2}(x_j\rho)
-
x_j(-\Delta)^{\alpha/2}\rho
=
-\alpha\partial_j(-\Delta)^{\alpha/2-1}\rho
\end{equation}
and stacking the components gives the vector form in equation~\eqref{eq:frac_lapl_identity}.

\end{proof}


\begin{proof}[Proof of Theorem~\ref{thm:fna_pde}]

By linearity in the attention heads, it suffices to treat one head. We omit the head index $h$ from $\mathbf{U}_h$, $\mathbf{W}_{\mathrm{Q,K}}^{(h)}$, $\mathbf{u}_h$, $\tilde\rho_h$ and $m_{\varepsilon,h}$ until we sum over heads.
Starting from the definition of the FNA kernel,
\begin{widetext}
\begin{equation}
\int_{\R^d} k_\varepsilon(\xb,\yb)\,\Ub\,\yb\,\rho(\yb)\d\yb
= \frac{\displaystyle\int_{\R^d}\Phi_\alpha\big(\norm{\Wb_{\mathrm{Q}}\xb-\Wb_{\mathrm{K}}\yb}/\varepsilon^{1/\alpha}\big)\,\Ub\,\yb\,\rho(\yb)\d\yb}
{\displaystyle\int_{\R^d} \Phi_\alpha\big(\norm{\Wb_{\mathrm{Q}}\xb-\Wb_{\mathrm{K}}\yb}/\varepsilon^{1/\alpha}\big)\,\rho(\yb)\d\yb}.
\end{equation}
Changing variables to $\zb = \Wb_{\mathrm{K}}\yb$ in both integrals cancels the constant Jacobian factor $\abs{\det\Wb_{\mathrm{K}}}^{-1}$ and replaces $\rho(\yb)$ with
$\tilde\rho(\zb)=\rho\big(\Wb_{\mathrm{K}}^{-1}\zb\big)$. 
With $\mathbf{u}\coloneqq\Wb_{\mathrm{Q}}\xb$, define $m_{\varepsilon}(\mathbf{u})$ as the kernel-weighted mean of the projected keys. Lemma~1 in the main text (``Methods'') gives the following locally uniform expansion:
\begin{align}\label{eq:head_mean}
\begin{aligned}
m_{\varepsilon}(\mathbf{u}) 
&\coloneqq
\int_{\R^d} k_\varepsilon(\xb,\yb)\,\Wb_{\mathrm{K}}\,\yb\,\rho(\yb)\d\yb, \\
&= \frac{c_0\,\mathbf{u}\,\tilde\rho(\mathbf{u}) - \kappa_\alpha\varepsilon\,(-\Delta)^{\alpha/2}(\zb\tilde\rho)(\mathbf{u}) + o(\varepsilon)}{c_0\,\tilde\rho(\mathbf{u}) - \kappa_\alpha\varepsilon\,(-\Delta)^{\alpha/2}\tilde\rho(\mathbf{u}) + o(\varepsilon)}, \\
&= \mathbf{u} \;-\; \varepsilon\,\tilde c_\alpha\,
\frac{(-\Delta)^{\alpha/2}(\zb\,\tilde\rho)(\mathbf{u}) - \mathbf{u}\,(-\Delta)^{\alpha/2}\tilde\rho(\mathbf{u})}{\tilde\rho(\mathbf{u})} + o(\varepsilon).
\end{aligned}
\end{align}
\end{widetext}
Under Assumption~\ref{ass:main}(\hyperref[ass:main]{ii}), $\Ub\Wb_{\mathrm{K}}^{-1} = -\Wb_{\mathrm{Q}}^{\top}$ and hence
$\Ub\Wb_{\mathrm{K}}^{-1}\Wb_{\mathrm{Q}} = -\Wb_{\mathrm{Q}}^{\top}\Wb_{\mathrm{Q}}$.
Applying Lemma~\ref{lemma:comm_relation} to the second term of the expansion above gives
\begin{widetext}
\begin{equation}
\Ub\,\Wb_{\mathrm{K}}^{-1}\, m_{\varepsilon}(\mathbf{u})
= -\Wb_{\mathrm{Q}}^{\top}\Wb_{\mathrm{Q}}\,\xb
\;-\;\varepsilon\,\alpha\,\tilde c_\alpha\,\Wb_{\mathrm{Q}}^{\top}\,
\frac{\big[\nabla(-\Delta)^{\alpha/2-1}\tilde\rho\big](\Wb_{\mathrm{Q}}\xb)}{\tilde\rho(\Wb_{\mathrm{Q}}\xb)}
+ o(\varepsilon).
\end{equation}
Assumption~\ref{ass:main}(\hyperref[ass:main]{iii}) ensures that the denominator $\tilde{\rho}(\Wb_{\mathrm{Q}}\xb)$ is positive.
We restore the head indices and substitute this expansion into the recentered field (equation~\eqref{eq:recentered_field}), with each head weighted by $\beta_h$:
\begin{equation}
\widetilde{T}_{\rho,0}(\mathbf{x})
= \lim_{\varepsilon\downarrow0}\widetilde{T}_{\rho,\varepsilon}(\xb)
= \lim_{\varepsilon \downarrow 0}\sum_{h=1}^H  \frac{\beta_h}{\varepsilon} \left[ \Ub(\Wb_{\mathrm{K}}^{(h)})^{-1} m_{\varepsilon}(\mathbf{u}_h) + \bigl(\WQ\bigr)^{\!\top}\WQ\,\xb \right] 
\end{equation}
\end{widetext}
Recentering cancels the linear term proportional to $\bigl(\WQ\bigr)^{\!\top}\WQ\,\xb$, and the $o(\varepsilon)$ remainder tends to zero after division by $\varepsilon$.
We conclude the proof by taking the limit $\varepsilon \downarrow 0$.

\end{proof}

The following corollary specializes Theorem~\ref{thm:fna_pde} to identity query and key projections and equal head weights. 

\begin{corollary} \label{cor:fna_to_heat}

Under the assumptions of Theorem~\ref{thm:fna_pde}, if $\mathbf{W}_{\mathrm{Q}}^{(h)} = \mathbf{W}_{\mathrm{K}}^{(h)} = I_{d}$ and $\beta_h = 1/H$ for each $h$, the limiting multi-head FNA continuity equation is the fractional diffusion equation
\begin{equation}  \label{eq:frac_heat}
\partial_t \rho + c_{\alpha} (-\Delta)^{\alpha/2} \rho = 0,
\end{equation}
with the constant
\begin{equation}
\begin{aligned} \label{eq:constants_2}
&c_\alpha\\
&=
\begin{cases}
\dfrac{\alpha}{C_{d,\alpha}}\Big(\dfrac{2\pi^{d/2}}{\Gamma(d/2)}\cdot\dfrac{\Gamma(d)\Gamma(\alpha)}{\Gamma(d+\alpha)}\Big)^{-1},
& 1 \leq \alpha<2,\\[3mm]
\tfrac12, & \alpha=2 .
\end{cases}
\end{aligned}
\end{equation}
Here, $c_{\alpha}$, $\tilde{c}_{\alpha}$ are defined in ``Methods'' of the main text (equation~(36)) and $C_{d,\alpha}$ in equation~(8); $\Gamma(\cdot)$ is the Gamma function. 
    
\end{corollary}

\begin{proof}[Proof of Corollary~\ref{cor:fna_to_heat}]

By substituting the identity projections and $\beta_h=1/H$ into equation~\eqref{eq:fna_field_limit_comm}
\begin{equation}
\begin{aligned}
&\nabla\cdot\big(\rho\,\widetilde{T}_{\rho,0}\big)\\
&= -\alpha\,\tilde c_\alpha\,\nabla\cdot\nabla(-\Delta)^{\alpha/2-1}\rho
= -\alpha\,\tilde c_\alpha\,\Delta(-\Delta)^{\alpha/2-1}\rho\\
&= c_\alpha\,(-\Delta)^{\alpha/2}\rho.
\end{aligned}
\end{equation}
Thus, equation~\eqref{eq:limit_pde} reduces to equation~\eqref{eq:frac_heat} with $c_\alpha=\alpha\tilde c_\alpha
= \alpha\kappa_\alpha/c_0$ as given in equation~\eqref{eq:constants_2}. 
For $1 \leq \alpha<2$ insert $\kappa_\alpha=C_{d,\alpha}^{-1}$ and the
closed form of $c_0$ is also provided by equation~(36) of the main text. 
For $\alpha=2$, $\tilde c_2=1/4$ gives
$c_2=1/2$.
    
\end{proof}

\section{Stochastic dynamics, geometry and distribution underlying diffusion maps \label{supp_sec:dm_algorithn_continuous}}

\subsection{Diffusion regimes}  \label{supp_sub_sec:diffusion_regimes}

The infinitesimal generator $\mathcal{H}$ in equation~(20) of Theorem~2 in the main text depends on the probability density $\rho$.
Here, we outline a more general construction of the attention weights $\mathbf{A}$ for the Gaussian case ($\alpha = 2$), following the diffusion maps framework \cite{Coifman2006,Nadler2008}.
Here, we continue to work under Assumptions~A and~B of the main text and consider single-head FNA with $\mathbf{W}_{\mathrm{Q}} = \mathbf{W}_{\mathrm{K}} = \mathbf{I}_d$.
Construct the diagonal matrix $D_{ii} = \sum_{j} C_{ij}$ and the matrix
\begin{equation} \label{eq:normalized_diffusion_matrix}
\mathbf{C}^{(a)} = \mathbf{D}^{-a}\mathbf{C}\mathbf{D}^{-a}.
\end{equation}
After the density normalization with parameter $a$ in equation~\eqref{eq:normalized_diffusion_matrix}, we row-normalize $\mathbf{C}^{(a)}$ to obtain the transition matrix:
\begin{align} 
D_{ii}^{(a)} &= \sum_{j=1}^n C_{ij}^{(a)}, \label{eq:normalized_kernel_deg} \\
\mathbf{A}^{(a)} &= \rbrac{\mathbf{D}^{(a)}}^{-1} \mathbf{C}^{(a)} \equiv \rownorm(\mathbf{C}^{(a)}). \label{eq:MC}
\end{align}   
The attention weights in equation~(16) of the main text correspond to $\mathbf{A}^{(0)}$ with the superscript omitted.
In the infinite-sequence limit $n \rightarrow \infty$, the Monte Carlo average converges to the corresponding integral:
\begin{equation}
\begin{aligned}
&\lim_{n \rightarrow 
\infty} \frac{1}{n} \sum_{j=1}^n c_{\varepsilon}(\mathbf{x}_i, \mathbf{x}_j) f(\mathbf{x}_j) \\
&= \int_{\R^d} c_{\varepsilon}(\mathbf{x}_i,\mathbf{y}) f(\mathbf{y}) \rho(\mathbf{y})  \d \mathbf{y},
\end{aligned}
\end{equation}
where $\rho$ is the underlying sampling probability density.

The setup in equations~\eqref{eq:normalized_diffusion_matrix}--\eqref{eq:MC} extends equations~(15)--(16) from the main text.

We express the continuum counterpart of this construction through the following integrals in equations~\eqref{eq:r_invariant_kernel}--\eqref{eq:transition_kernel}.
\begin{enumerate}
    \item Define the rotation-invariant kernel for all pairs of points $\mathbf{x}, \mathbf{y} \in \mathbb{R}^d$ by
    \begin{equation} \label{eq:r_invariant_kernel}
    c_{\varepsilon}(\mathbf{x},\mathbf{y}) = \Phi_{\alpha}\rbrac{\frac{\lVert \mathbf{x}-\mathbf{y}\rVert}{\sqrt{\varepsilon}}}.
    \end{equation}
    \item Let 
    \begin{equation} \label{eq:q_estimate}
    \rho_{\varepsilon}(\mathbf{x}) \coloneqq \int_{\R^d} c_{\varepsilon}(\mathbf{x},\mathbf{y}) \rho(\mathbf{y}) \d \mathbf{y}
    \end{equation}
    and define the density-normalized kernel
    \begin{equation} \label{eq:normalized_kernel}
    c_{\varepsilon}^{(a)}(\mathbf{x}, \mathbf{y}) \coloneqq \frac{c_{\varepsilon}(\mathbf{x}, \mathbf{y})}{\rho_{\varepsilon}(\mathbf{x})^a \rho_{\varepsilon}(\mathbf{y})^a} .
    \end{equation}     
    \item Normalize this kernel to obtain a transition kernel by setting
\begin{equation} \label{eq:normalized_kernel_deg_cts}
d_{\varepsilon}^{(a)}(\mathbf{x}) \coloneqq \int_{\mathbb{R}^d} c_{\varepsilon}^{(a)}(\mathbf{x}, \mathbf{y}) \rho(\mathbf{y}) \d \mathbf{y}
\end{equation}
and by defining the anisotropic transition kernel
\begin{equation} \label{eq:transition_kernel}
k_{\varepsilon}^{(a)}(\mathbf{x}, \mathbf{y}) \coloneqq \frac{c_{\varepsilon}^{(a)}(\mathbf{x}, \mathbf{y})}{d_{\varepsilon}^{(a)}(\mathbf{x})} .
\end{equation}   
\end{enumerate}
In equations~\eqref{eq:normalized_diffusion_matrix}--\eqref{eq:MC}, the choice of normalization parameter $a$ affects the operator limit of $k_{\varepsilon}^{(a)}$ as $\varepsilon \downarrow 0$ by controlling the trade-off between the geometry of the data and the sampling density $\rho$.

Now let us define the following attention map which parallels equation~(18) in the main text:
\begin{equation}
T_{\varepsilon, a} [f] (\mathbf{x}) \coloneqq \int_{\R^d} k_{\varepsilon}^{(a)} (\mathbf{x}, \mathbf{y}) f(\mathbf{y}) \rho(\mathbf{y}) \d \mathbf{y}.
\end{equation}
First, we focus on the Gaussian case where $\alpha = 2$.
The following limit is adapted from ref.~\cite{Coifman2006}. 
\smallskip

\begin{theorem}

Define the infinitesimal generator of the Markov operator $T_{\varepsilon,a}$:
\begin{equation}
\mathcal{H}_{\varepsilon, a} = \frac{T_{\varepsilon, a} - I}{\varepsilon},
\end{equation}
where $I$ is the identity operator.
It follows that 
\begin{equation}
\begin{aligned}
&\lim _{\varepsilon \downarrow 0} \mathcal{H}_{\varepsilon, a} f\\
&= \frac{1}{4} \left[ \frac{\Delta\left(f \rho^{1-a}\right)}{\rho^{1-a}}-\frac{\Delta\left(\rho^{1-a}\right)}{\rho^{1-a}} f \right].
\end{aligned}
\end{equation}

\end{theorem}

For $a = 1$, the dependence on the sampling density disappears in the limit, yielding $\lim _{\varepsilon \downarrow 0} \mathcal{H}_{\varepsilon, a} = \Delta/4$.
In our construction, we set $a = 0$ in equation~(16) of the main text, as this keeps FNA directly comparable to standard attention and retains the dependence on the token distribution.

\subsection{Density normalization for multiscale FNA}  \label{supp_sub_sec:density_normalization}

In the Gaussian case, setting $a = 1$ yields a constant multiple of the Laplacian.
Here, we consider the multiscale case ($\alpha < 2$) on $\R^d$ under the assumptions of Theorem~2 in the main text.
The embedding density $\rho$ appears in the infinitesimal generator in equation~(20) of the main text. 
Following the fractional diffusion maps construction \cite{Antil2021}, we can remove this dependence by normalizing the attention scores as follows.
\begin{enumerate}
    \item Similar to equation~(15) of the main text, define
    \begin{equation} 
    C_{ij} = \Phi_{\alpha}\rbrac{ \frac{\Vert \mathbf{W}_{\mathrm{Q}} \mathbf{x}_i - \mathbf{W}_{\mathrm{K}} \mathbf{x}_j \Vert}{\sqrt{\varepsilon}}}.
    \end{equation}
    \item Then, we define a diagonal matrix 
    \begin{align}
    D_{ii} \coloneqq \frac{(2 \pi \varepsilon)^{-d / 2}}{n} \sum_{j=1}^n \exp \left(-\frac{\Vert \mathbf{W}_{\mathrm{Q}} \mathbf{x}_i - \mathbf{W}_{\mathrm{K}} \mathbf{x}_j \Vert^2}{2 \varepsilon}\right)
    \end{align}
    and $\mathbf{C}^{\alpha,(1)} = \mathbf{D}^{-1} \mathbf{C} \mathbf{D}^{-1}$.
    \item Finally $\mathbf{A} \coloneqq \rownorm(\mathbf{C}^{\alpha,(1)})$ through row-normalization and the density-normalized FNA attention scores are $\FAttn(\mathbf{X}) \coloneqq \mathbf{V} \rownorm(\mathbf{C}^{\alpha, (1)})^{\top}$.
\end{enumerate}
Compared with equations~(15)--(16) in the main text, step 2 adds the Gaussian density estimate. 
The resulting normalization removes the density dependence from the Euclidean limit under Assumptions~A and~B.
We defer the technical details to the original work \cite{Antil2021}. Section~\ref{supp_sec:spherical_fna} discusses the spectral comparison for the geodesic kernel on the circle.

\begin{figure*}[!t]
\centering
\includegraphics[scale=0.9]{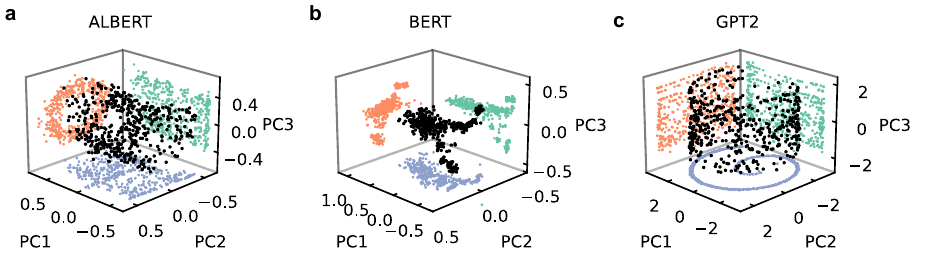}
\caption{\textbf{Principal component analysis of token embeddings.} \label{fig:pretrained_large_model_embeddings} 
\textbf{a}--\textbf{c} Black points show projections onto the first three principal components of the token embeddings from pretrained ALBERT, BERT and GPT-2, respectively. All three models received the same sentence generated by GPT-2.
PC1, PC2 and PC3 denote the first, second and third principal components. Colored points show the projections onto PC1 and PC2 (blue), PC1 and PC3 (green), and PC2 and PC3 (orange).
}
\end{figure*}

\subsection{Distribution of token embeddings \label{supp_sub_sec:pretrained_model_embeddings}}

To highlight the non-uniformity of $\rho$ in Theorem~2 of the main text, we examine the embeddings of an identical input sequence in three different, popular, pre-trained Transformer models: ALBERT \cite{Lan2020}, BERT \cite{Devlin2018} and GPT2 \cite{Radford2019}. In Fig.~\ref{fig:pretrained_large_model_embeddings}\hyperref[fig:pretrained_large_model_embeddings]{a}, \hyperref[fig:pretrained_large_model_embeddings]{c}, ALBERT and GPT2 reveal circular and spiral patterns, respectively. 
These geometrical structures have previously been associated with the representation of token position, with the positional component of Transformer hidden states exhibiting low-dimensional, continuous, and often circular or spiral geometries \cite{Song2026}. 
In contrast, the BERT embeddings form clusters (Fig.~\ref{fig:pretrained_large_model_embeddings}\hyperref[fig:pretrained_large_model_embeddings]{b}), consistent with prior observations of clustering in BERT contextual embeddings \cite{Cai2021}.
The concentration of embeddings along these structures illustrates spatial variation in $\rho$.

\section{FNA on the sphere}  \label{supp_sec:spherical_fna}

\begin{figure}[!htbp]
    \centering
    \includegraphics[scale=0.65]{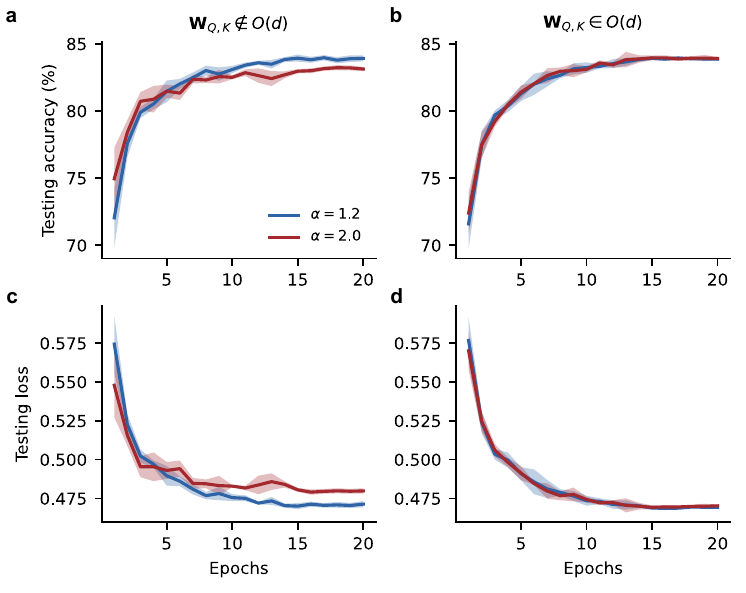}
    \caption{\textbf{Learning curves for text classification on IMDb.} \textbf{a} Mean test accuracy across five trials for multiscale FNA ($\alpha = 1.2$) and local attention ($\alpha = 2$) using $\mathbf{W}_{\mathrm{Q,K}} \notin O(d)$. 
    The blue and red curves show multiscale FNA (FNA, $\alpha=1.2$, blue) and local attention ($\alpha=2$, red), respectively. The shaded regions show one standard deviation across the five trials. $\mathbf{W}_{\mathrm{Q,K}}$ denotes the query and key projection matrices, and $O(d)$ is the orthogonal group in dimension $d$. 
    \textbf{b} Same as \textbf{a}, but with $\mathbf{W}_{\mathrm{Q,K}} \in O(d)$. 
    \textbf{c} Same as \textbf{a}, but showing test loss. 
    \textbf{d} Same as \textbf{c}, but with $\mathbf{W}_{\mathrm{Q,K}} \in O(d)$.}
    \label{fig:train_imdb_spherical}
\end{figure}

\begin{table}[!htbp]
\centering
\begin{tabular}{l|c|c|c|c}
\specialrule{.1em}{.05em}{.05em}  
Model & Mean & Median & Best & Worst \\
\hline
FNA & 83.83\% & \textbf{83.91\%} & \textbf{84.06\%}  & 83.41\% \\
LA & 83.08\% & 83.12\% & 83.27\% & 82.86\% \\
\hline 
op-FNA & \textbf{83.86\%} & 83.85\% & 84.02\% & \textbf{83.75\%} \\
op-LA & 83.83\% & 83.90\% & 83.96\% & 83.66\% \\
\specialrule{.1em}{.05em}{.05em} 
\end{tabular}
\vspace{0.2cm}
\caption{\label{table:train_imdb_spherical} 
Test accuracy on IMDb dataset across 5 independent training trials.
We use the setting $\manifold \coloneqq \sphere^{d - 1}$.
FNA, fractional neural attention ($\alpha=1.2$). LA, local attention ($\alpha=2$). The prefix op- denotes orthogonal query and key projections.
Bold marks the highest value in each column across the four models.
}
\end{table}

Through applying an explicit $L^2$-normalization to the embeddings $\mathbf{X}$, one can naturally project them onto the unit sphere $\sphere^{d_h - 1}$ for each attention head.
In this spherical setting, the Euclidean distance $\Vert \mathbf{q}_i - \mathbf{k}_j \Vert$ in equation~(15) is replaced by the geodesic distance on the sphere:
\begin{equation}
\begin{aligned}
&d_g(\iota_{\mathrm{Q}}(\mathbf{x}_i),\iota_{\mathrm{K}}(\mathbf{x}_j))\\
&\coloneqq \arccos(\iota_{\mathrm{Q}}(\mathbf{x}_i)^{\top} \iota_{\mathrm{K}}(\mathbf{x}_j)).
\end{aligned}
\end{equation}
Note that the isometric embedding will generally not hold if $\mathbf{W}_{\mathrm{Q,K}} \notin O(d)$ or $H > 1$. 
In order to guarantee the well-posedness of $d_g(\iota_{\mathrm{Q}}(\mathbf{x}_i),\iota_{\mathrm{K}}(\mathbf{x}_j))$, we define
\begin{equation}
\iota_{\mathrm{Q}}(\mathbf{x}) = 
\begin{cases}    
\mathbf{W}_{\mathrm{Q}} \mathbf{x} & \text{ if } \mathbf{W}_{\mathrm{Q}} \in O(d) \text{ and } H = 1 \\
\frac{ \mathbf{W}_{\mathrm{Q}} \mathbf{x}}{\Vert \mathbf{W}_{\mathrm{Q}} \mathbf{x} \Vert} & \text{ otherwise.} 
\end{cases}
\end{equation}
The key mapping $\iota_{\mathrm{K}}(\mathbf{x})$ is defined in the same fashion
Now, the spherical FNA score for $\alpha < 2$ is instead defined by
\begin{equation} \label{eq:spherical_frac_attn_score}
C_{i,j} \coloneqq \Phi_{\alpha} \left( \frac{d_g(\iota_{\mathrm{Q}}(\mathbf{x}_i), \iota_{\mathrm{K}}(\mathbf{x}_j))}{\kappa}  \right).
\end{equation}
To compensate for floating-point rounding errors, we clamp the query--key dot products $\iota_{\mathrm{Q}}(\mathbf{x}_i)^{\top}\iota_{\mathrm{K}}(\mathbf{x}_j)$ to $[-1,1]$ before evaluating $\arccos$.
Let $d_{\mathrm{pow}}$ denote the dimension parameter used in the power-law profile from equation~(11) of the main text, whose exponent is $d_{\mathrm{pow}}+\alpha$. We set $d_{\mathrm{pow}}=d-1$ for a single head and $d_{\mathrm{pow}}=d_h$ for multiple heads. 
We set $\kappa=\pi/(\pi^{1/d_{\mathrm{pow}}}-1)$ for $\alpha<2$, and $\kappa=1$ for $\alpha=2$. 
Assigning padded keys the maximum scaled distance \(\pi/\kappa\) does not by itself exclude them from attention, because the kernel remains positive at this finite distance. A separate padding mask is therefore needed to set their attention weights to zero.

For $\alpha<2$ in the uniform-circle experiment (Fig.~2a of the main text), we use the geodesic power-law profile from equation~\eqref{eq:spherical_frac_attn_score} with the bandwidth specified in the main text. 
The estimated generator decay rates exhibit power-law scaling consistent with that of the spectral fractional Laplacian on the circle over the numerically resolved range. Theorem~2 in the main text establishes the generator limit in Euclidean space.

As in Fig.~7 of the main text, we evaluate models with six layers and eight attention heads on the IMDb dataset. Here we use the spherical variant of FNA. 
The initial learning rate is set to $1.2 \times 10^{-4}$ and decayed by a factor of 10 at epoch 15.
The corresponding training curves and performance results are shown in Fig.~\ref{fig:train_imdb_spherical} and Supplementary Table~\ref{table:train_imdb_spherical}, respectively.

\onecolumngrid
\section{Additional experimental results}

\subsection{Text classification}
\label{supp_sub_sec:text_classification}

\textbf{Distance-based attention ablation.}
To complement the random masking experiment in Fig.~4 of the main text, we ablate query--key connections at inference in trained single-layer, single-head IMDb models. For each distance threshold, we retain connections whose query--key distance is at or below the threshold (Fig.~\ref{fig:dist_dynamic_inference_all}).  
At low thresholds, the accuracy of both multiscale FNA and local attention falls to near chance levels. 
Multiscale FNA recovers its accuracy at lower distance thresholds than local attention across the embedding dimensions shown.
With tied identity projections, dot-product attention (DP) and Sinkformer (SINK) retain substantial accuracy at low thresholds for $d\geq 16$ and improve only modestly as more connections are restored. Their accuracy is more sensitive to the threshold when the query and key projections are distinct.

\par\noindent\begin{minipage}{\textwidth}

    \centering
    \includegraphics[scale=1]{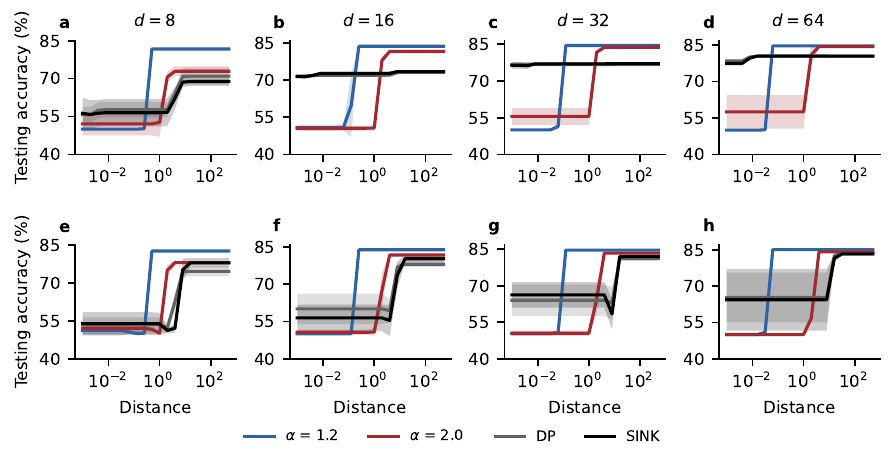}
    \captionof{figure}{\textbf{Distance-based ablation of query--key connections on IMDb.}
    Mean test accuracy as a function of the distance threshold. Query--key connections beyond the threshold are masked. Shading shows one standard deviation.
    \textbf{a}--\textbf{d} Tied identity projections ($\mathbf{W}_{\mathrm{Q}}=\mathbf{W}_{\mathrm{K}}=\mathbf{I}_d$) with embedding dimensions $d\in\{8,16,32,64\}$, respectively.
    \textbf{e}--\textbf{h} Same embedding dimensions with $\mathbf{W}_{\mathrm{Q}}=\mathbf{I}_d$ and $\mathbf{W}_{\mathrm{K}}\in O(d)$.
    Blue, multiscale fractional neural attention (FNA, $\alpha=1.2$). Red, local attention ($\alpha=2$). Grey, dot-product attention (DP). Black, Sinkformer (SINK). 
    $O(d)$ is the orthogonal group in dimension $d$.}
    \label{fig:dist_dynamic_inference_all}
\end{minipage}\par\medskip


\par\noindent\begin{minipage}[t]{.48\textwidth}\vspace{0pt}
\textbf{AG News classification.}
We compare the attention mechanisms on four-class topic classification using models with six layers, eight attention heads and an embedding dimension $d=256$. We train for 20 epochs with a batch size of 32 and an initial learning rate of $10^{-4}$, which we reduce by a factor of 10 at epoch 15. Multiscale FNA attains the highest median test accuracy at the end of training with both unconstrained and orthogonal query and key projections (Fig.~\ref{fig:train_agnews}). \\

\textbf{Training set size on IMDb.}
We train single-layer, single-head models on subsets comprising 10\% to 50\% of the full IMDb dataset. 
Each model is trained for 20 epochs with a batch size of 16 and an initial learning rate of $10^{-3}$, which we reduce by a factor of 5 at epoch 15. Test accuracy generally decreases as the training set shrinks (Fig.~\ref{fig:train_imdb_trainset_pct_v2}). The gap between multiscale FNA and local attention is most apparent at smaller embedding dimensions, especially with the smallest training subsets. Their curves nearly coincide at $d=64$.
\end{minipage}\hfill
\begin{minipage}[t]{.48\textwidth}\vspace{0pt}

    \centering
    \includegraphics[scale=0.65]{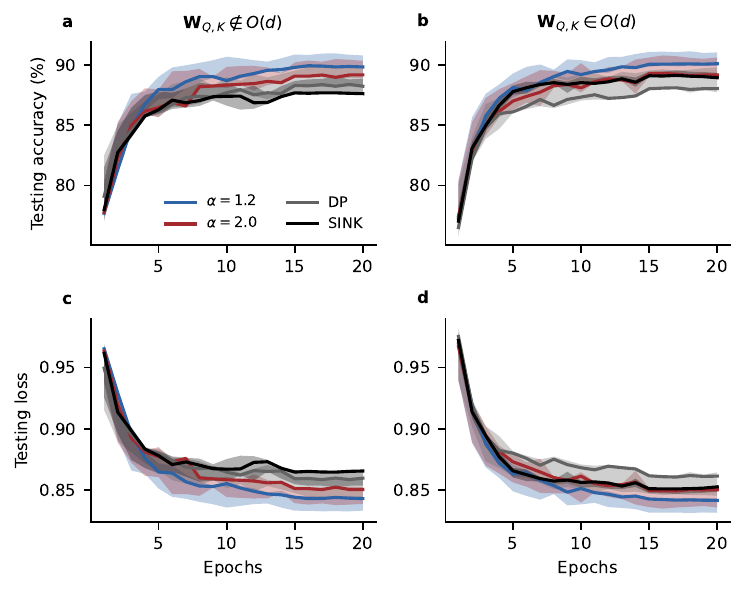}
    \captionof{figure}{\textbf{Learning curves for text classification on AG News.}
    \textbf{a} Median test accuracy across five trials with unconstrained query and key projection matrices. Shading spans the 25th to 75th percentiles across trials at each epoch.
    \textbf{b} Same as \textbf{a}, with orthogonal projections ($\mathbf{W}_{\mathrm{Q,K}}\in O(d)$).
    \textbf{c},\,\textbf{d} Median test loss for the projection settings in \textbf{a} and \textbf{b}, respectively, with the same percentile shading.
    }
    \label{fig:train_agnews}
\end{minipage}\par\medskip

\par\noindent\begin{minipage}{\textwidth}

    \centering
    \includegraphics[scale=0.83]{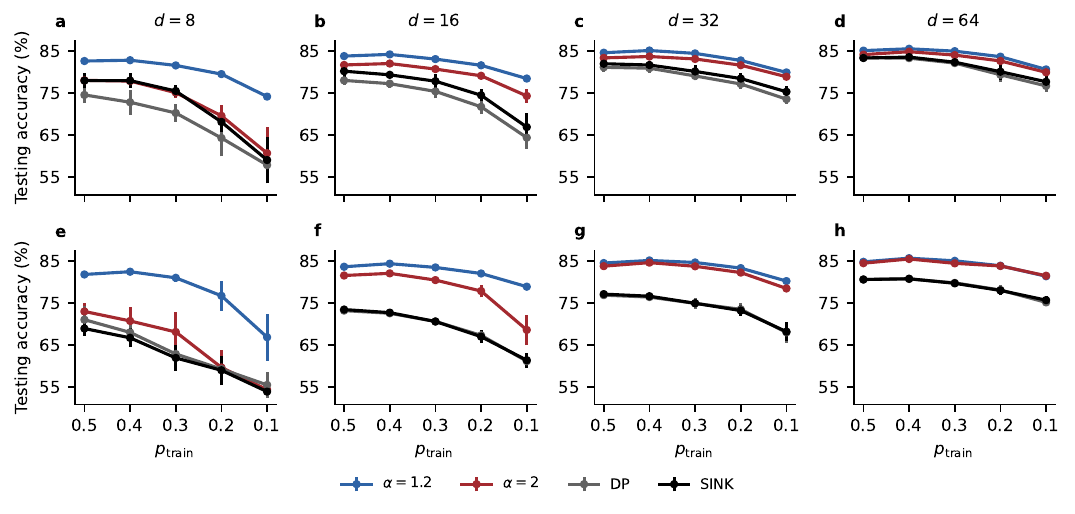}
    \captionof{figure}{\textbf{Effect of training set size on IMDb test accuracy.}
    Mean test accuracy as a function of the fraction $p_{\mathrm{train}}$ of the full 50,000-review IMDb dataset used for training.
    For $p_{\mathrm{train}}=0.1$ to $0.4$, training and evaluation sets are drawn from seeded, disjoint splits of the pooled labeled corpus. The $p_{\mathrm{train}}=0.5$ reference uses the official split. Error bars show one sample standard deviation across runs.
    \textbf{a}--\textbf{d} Distinct query and key projections ($\mathbf{W}_{\mathrm{Q}}=\mathbf{I}_d$, $\mathbf{W}_{\mathrm{K}}\in O(d)$) with embedding dimensions $d=8,16,32,64$, respectively.
    \textbf{e}--\textbf{h} Same embedding dimensions with tied identity projections ($\mathbf{W}_{\mathrm{Q}}=\mathbf{W}_{\mathrm{K}}=\mathbf{I}_d$).
    }
    \label{fig:train_imdb_trainset_pct_v2}
\end{minipage}\par\medskip

\par\noindent\begin{minipage}[t]{.48\textwidth}\vspace{0pt}
\textbf{Context-window length on IMDb.}
We vary the maximum training sequence length $n_{\max}$ from 128 to 1,024 tokens in models with six layers, eight attention heads and $d=256$. We train for 20 epochs with an initial learning rate of $10^{-4}$, reduced by a factor of 10 at epoch 15. The batch size is 32 for windows up to 512 tokens and 8 for the 1,024-token window. Over the settings tested, longer windows are associated with higher mean test accuracy and lower mean test loss for all four attention mechanisms. Multiscale FNA has the highest mean accuracy and lowest mean loss at each window length under both projection settings (Fig.~\ref{fig:imdb_window_effect}). The comparison between 512 and 1,024 tokens also includes the change in batch size. \\

\textbf{Sensitivity to the fractional order.}
We sweep $\alpha$ in IMDb models with six layers, eight attention heads and $d=256$ trained for 20 epochs. For each context-window length, mean test accuracy varies little over the tested multiscale settings ($\alpha<2$) and decreases at the local setting ($\alpha=2$) under both projection settings (Fig.~\ref{fig:fna_nlp_alpha_effects}). 
\end{minipage}\hfill
\begin{minipage}[t]{.48\textwidth}\vspace{0pt}

    \centering
    \includegraphics[scale=0.65]{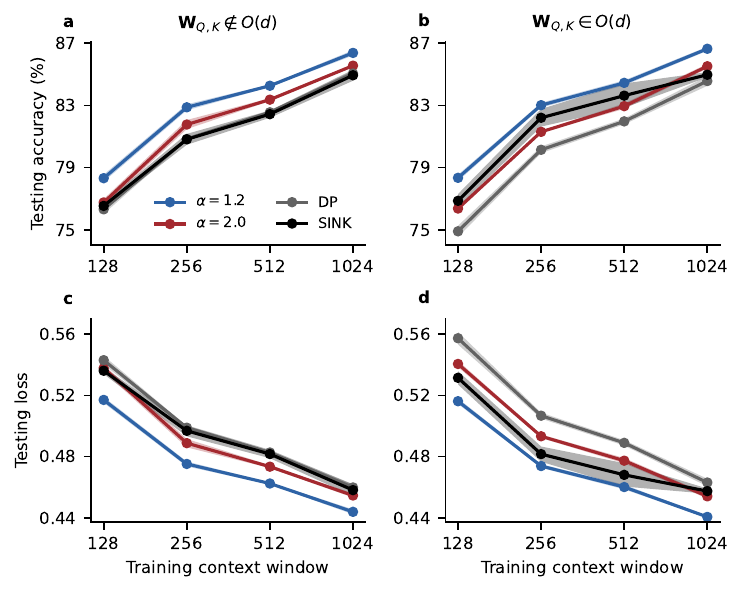}
    \captionof{figure}{\textbf{Effect of training context-window length on IMDb classification.}
    \textbf{a} Mean test accuracy at the end of training across five trials for each maximum training sequence length $n_{\max}$ with unconstrained query and key projections. Shading shows one standard deviation across trials.
    \textbf{b} Same as \textbf{a}, with orthogonal projections ($\mathbf{W}_{\mathrm{Q,K}}\in O(d)$).
    \textbf{c},\,\textbf{d} Mean test loss for the projection settings in \textbf{a} and \textbf{b}, respectively, with the same shading.
    }
    \label{fig:imdb_window_effect}
\end{minipage}\par\medskip

\par\noindent\begin{minipage}{\textwidth}

    \centering
    \includegraphics[scale=1]{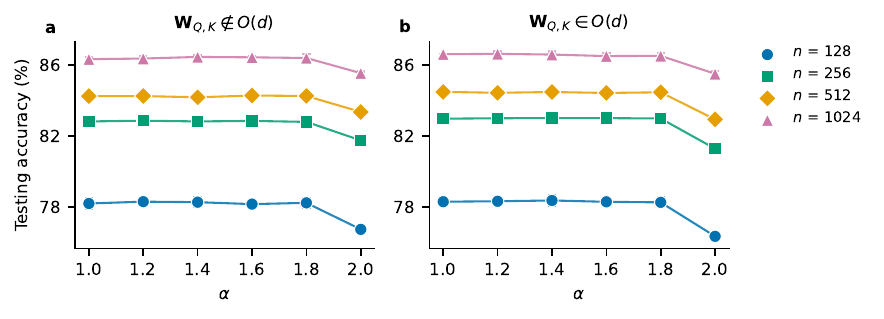}
    \captionof{figure}{\textbf{Effect of fractional order on IMDb test accuracy.}
    \textbf{a} Mean test accuracy across five independent runs as a function of the fractional order $\alpha$, with unconstrained query and key projections. Context-window lengths $n=128,256,512,1024$ are shown as blue circles, green squares, orange diamonds and pink triangles, respectively.
    \textbf{b} Same as \textbf{a}, with orthogonal query and key projections ($\mathbf{W}_{\mathrm{Q,K}}\in O(d)$). 
    }
    \label{fig:fna_nlp_alpha_effects}
\end{minipage}\par\medskip


\par\noindent\begin{minipage}[t]{.48\textwidth}\vspace{0pt}
\subsection{Neural machine translation}  \label{supp_sub_sec:translation}

\textbf{Multi30K machine translation.}
We compare learning curves for English-to-German translation using models with four encoder layers, four decoder layers and eight attention heads. We train for 40 epochs using Adam with a cosine learning-rate schedule and 1,000 warm-up steps. Fig.~\ref{fig:train_translation} shows epochs 15--40. BLEU increases and test loss generally decreases over this interval for all four attention mechanisms. At epoch 40, multiscale FNA has the highest median BLEU with unconstrained projections. With orthogonal projections, its median BLEU is close to that of each of the strongest baselines.
\end{minipage}\hfill
\begin{minipage}[t]{.48\textwidth}\vspace{0pt}

    \centering
    \includegraphics[scale=0.65]{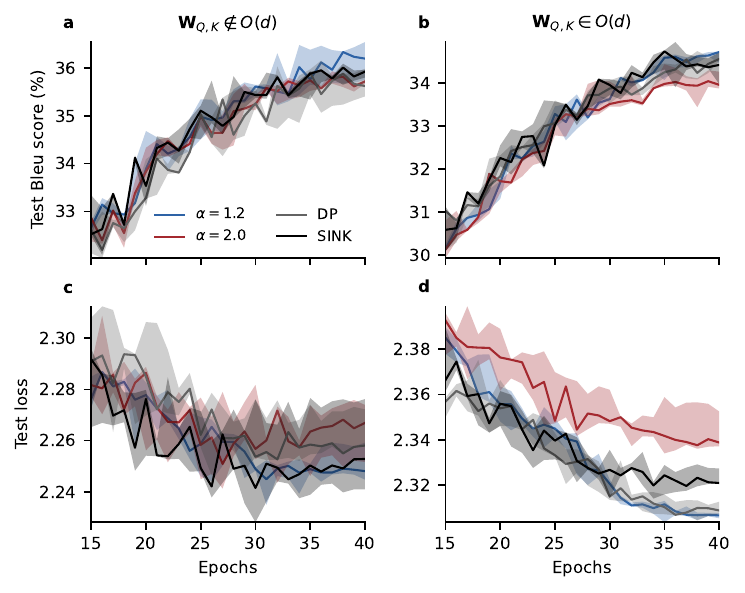}
    \captionof{figure}{\textbf{Learning curves for English-to-German translation on Multi30K.}
    \textbf{a} Median test BLEU (bilingual evaluation understudy) across five trials with unconstrained query and key projection matrices. Shading spans the 25th to 75th percentiles across trials at each displayed epoch.
    \textbf{b} Same as \textbf{a}, with orthogonal projections ($\mathbf{W}_{\mathrm{Q,K}}\in O(d)$).
    \textbf{c},\,\textbf{d} Median test loss for the projection settings in \textbf{a} and \textbf{b}, respectively, with the same percentile shading.
    }
    \label{fig:train_translation}
\end{minipage}\par\medskip

\par
\twocolumngrid
\bibliographystyle{unsrtnat-initials}
\bibliography{bibliography}

\end{document}